\documentclass{article}

\usepackage{amsmath,amsfonts,bm}

\def\eqref#1{equation~\ref{#1}}
\def\Eqref#1{Equation~\ref{#1}}

\def\1{\bm{1}}

\DeclareMathAlphabet{\mathsfit}{\encodingdefault}{\sfdefault}{m}{sl}
\SetMathAlphabet{\mathsfit}{bold}{\encodingdefault}{\sfdefault}{bx}{n}

\usepackage{amsmath}
\usepackage{mathtools}
\usepackage{amsthm}
\usepackage{xcolor}

\newtheorem{theorem}{Theorem}[section]
\newtheorem{lemma}[theorem]{Lemma}

\newtheorem{assumption}[theorem]{Assumption}
\newtheorem{definition}[theorem]{Definition}
\newtheorem{proposition}[theorem]{Proposition}
\newtheorem{corollary}[theorem]{Corollary}

\newcommand{\Rbb}{\mathbb{R}}

\newcommand{\Vcal}{\mathcal{V}}
\newcommand{\Scal}{\mathcal{S}}

\newcommand{\Lcal}{\mathcal{L}}
\newcommand{\Zcal}{\mathcal{Z}}
\newcommand{\Xcal}{\mathcal{X}}
\newcommand{\Ncal}{\mathcal{N}}
\newcommand{\Fcal}{\mathcal{F}}

\newcommand{\zsf}{\mathsf{z}}
\newcommand{\xsf}{\mathsf{x}}

\newcommand{\vep}{\varepsilon}
\newcommand{\Pdim}{\text{Pdim}}
\newcommand{\Ocal}[1]{\mathcal{O}\left(#1\right)}

\usepackage{hyperref}
\usepackage{url}

\usepackage{PRIMEarxiv}
\usepackage{natbib}
\usepackage[utf8]{inputenc} 
\usepackage[T1]{fontenc}    
\usepackage{hyperref}       
\usepackage{url}            
\usepackage{booktabs}       
\usepackage{amsfonts}       
\usepackage{nicefrac}       
\usepackage{microtype}      
\usepackage{lipsum}
\usepackage{fancyhdr}       
\usepackage{graphicx}       
\graphicspath{{media/}}     
\usepackage{subcaption} 
\usepackage{epstopdf}
\newcommand{\srankV}{r_{\text{s}}(V)}

\title{Towards better generalization: Weight Decay induces low-rank bias for neural networks
}
\author{Ke Chen \\
Department of Mathematical Science\\
University of Delaware\\
Newark, DE 19716, USA \\
\texttt{kechen@udel.edu} \\
\And
Chugang Yi  \\
Department of Mathematics \\
University of Maryland,  \\
College Park, MD 20742,USA \\
\texttt{chugang@umd.edu} \\
\And
Haizhao Yang \\
Department of Mathematics \\
University of Maryland,  \\
College Park, MD 20742,USA \\
\texttt{hzyang@umd.edu} \\
}

\begin{document}
\maketitle

\begin{abstract}
We study the implicit bias towards low-rank weight matrices when training neural networks (NN) with Weight Decay (WD). 
We prove that when a ReLU NN is sufficiently trained with Stochastic Gradient Descent (SGD) and WD, its weight matrix is approximately a rank-two matrix. 
Empirically, we demonstrate that WD is a necessary condition for inducing this low-rank bias across both regression and classification tasks. 
Our work differs from previous studies as our theoretical analysis does not rely on common assumptions regarding the training data distribution, optimality of weight matrices, or specific training procedures. 
Furthermore, by leveraging the low-rank bias, we derive improved generalization error bounds and provide numerical evidence showing that better generalization can be achieved.
Thus, our work offers both theoretical and empirical insights into the strong generalization performance of SGD when combined with WD.
\end{abstract}

\section{Introduction}

We consider the learning problem of a neural network
\begin{equation}\label{eqn:poploss}
    \min_{\theta\in\Theta } F(\theta) \coloneqq \mathbb{E}_z\left[l(z,\phi_\theta) \right] \,,
\end{equation}
where $\theta$ represents the parameter of the neural network $\phi_\theta$, $\Theta \subset \Rbb^d$ is the parameter space, $F:\Theta\rightarrow \Rbb$ is called the population loss, and $l(z,\phi_\theta)$ measures the loss function of $\phi_\theta$ over a random data point $z\in \Zcal$. The optimization problem above aims to fit the data distribution $\Zcal$ by adjusting the model parameter $\theta$. 
In practice, the data distribution $\Zcal$ is unknown thus the expectation over $z$ cannot be calculated. An optimization of the empirical loss function $\hat{F}:\Theta \rightarrow \Rbb$ is considered instead
\begin{equation}\label{eqn:emploss}
    \min_{\theta\in\Theta} \hat{F}(\theta)\coloneqq \frac{1}{N} \sum_{i=1}^N l(z_i,\phi_\theta) \,.
\end{equation}
Here $N$ is the total number of data points and $\zsf\coloneqq[z_1,\ldots,z_N] \in \Zcal^N$ are independent and identically distributed (i.i.d.) samples from $\Zcal$. 
A learning algorithm solving \Eqref{eqn:emploss} can be written as $A:\Zcal^N \rightarrow \Theta$ that maps the training data $\zsf$ to a parameter $A(\zsf)$. Let $\theta^\ast$ denote a minimizer of \Eqref{eqn:emploss}, the population loss of the algorithm $A$ can be decomposed into three parts:
\begin{equation}\label{eqn:decomp_error}
    F(A(\zsf)) = \underbrace{\left(F(A(\zsf)) - \hat{F}(A(\zsf))\right)}_\text{generalization error} + \underbrace{\left(\hat{F}(A(\zsf)) - \hat{F}(\theta^\ast)\right)}_\text{optimization error} + \underbrace{\hat{F}(\theta^\ast)}_\text{approximation error} \,.
\end{equation}
The first term measures the generalization error due to finite number of training data; the second term quantifies the gap between the empirical loss of the learned parameter $A(\zsf)$ and the minimum empirical loss; the last term indicates approximation power of the model hypothesis space $\{\phi_\theta:\theta\in \Theta\}$ in fitting the data distribution $\Zcal$.
When the dimension of the parameter space $\Theta$ increases, both the approximation error and optimization error decrease to zero due to universal approximation theory~\citep{hornik1991approximation,chen1995universal,yarotsky2022universal} and neural tangent kernel theory~\citep{jacot2018neural} respectively. 

Nevertheless, challenges remain that as classical generalization error analysis~\citep{anthony1999neural,feldman2016generalization} suggests that the generalization error for a neural network with $\mathcal{O}(d)$ parameters grows at the order of $\mathcal{O}(\sqrt{d})$, regardless of any learning algorithms.

However, empirically it is observed that SGD can produce implicit regularization~\citep{keskar2016large,zhu2018anisotropic,zhang2021understanding} and thus achieve better generalization performance than Gradient Descent (GD)~\citep{zhang2021understanding}. 
It is found that different learning algorithms exhibit implicit biases in the trained neural network, including alignment of neural network weights \citep{du2018algorithmic,ji2018gradient,ji2020directional,lyu2019gradient} and low-rank bias \citep{timor2023implicit,jacot2023implicit,xu2023dynamics}. Although these implicit biases are not theoretically understood yet,
the generalization error is expected to be improved by leveraging these implicit biases. As a consequence, one can obtain ``algorithm-dependent" generalization bounds that are usually smaller than the uniform bound.

In this work, we focus on the low-rank bias observed in many computer vision problems~\citep{yu2017compressing,alvarez2017compression,arora2018stronger}.
In particular, it is numerically shown that replacing the weight matrices by their low-rank approximation only results in small prediction error. The line of works~\citep{ji2020directional,lyu2019gradient} generalizes the theory of \citet{ji2018gradient} and proves the implicit bias effect of SGD algorithms with entropy loss. \citet{galanti2022sgd},\citet{timor2023implicit} and \citet{jacot2023implicit} first studied the low-rank bias for neural network trained with WD for square loss. They show that when minimizing the $l_2$ norm of coefficients of a neural network that attains zero loss, the coefficient matrices are of low-rank for \textit{sufficiently deep} neural networks. Similar low-rank bias is also empirically observed in other training settings, including Sharpness Aware Minimization~\citep{andriushchenko2023sharpness} and \textit{normalized weight decay}~\citep{xu2023dynamics}. 

\begin{table}[htbp]
\centering
\tiny
\caption{Assumptions and results of various papers on low-rank bias. The notation $M$ denotes an upper bound of the Frobenius norm of weight matrices, $B$ denotes the batch size, and $m_l$ denotes the number of patches in the input space.
Here GF denotes Gradient Flow and N/A means there is no specific assumption in the corresponding column. Our paper considers the most general setting and proves a strict low-rank result.}
\label{tab:relatedresults}
\begin{tabular}{ p{1.4cm}|p{1.5cm}|p{1.5cm}|p{1.5cm}|c|p{1.3 cm}|c }
\hline
 Paper & NN Architecture & Data & Optimization & Optimizer & Convergence & Result\\ \hline
 \citet{ji2020directional} & Linear NN & Linearly \newline separable & Entropy loss & GF & N/A & Each layer has rank $\leq 1$ \\ \hline
 \citet{le2022training} & Top $K$ linear layers & Linearly  \newline separable & Entropy loss & GF & N/A & Top $K$ layers have rank $\leq 1$ \\ \hline
 \citet{ergen2021revealing} & NN without bias & $1$-dimensional & Min $l_2$ with\newline data fitting & N/A & Global min & First layer rank $\leq 1$ \\ \hline
 \citet{ongie2022role} & NN with bias & $d$-dimensional & Min $l^2$ with \newline data fitting & N/A & Global min & Linear layers have rank $\leq d$ \\ \hline
 \citet{timor2023implicit} & $L$-layer NN without bias & Separable \newline  by $L'$-layer NN & Min $l_2$ with\newline data fitting & N/A & Global min & Each layer has rank $\leq M^{L/L'}$ \\ \hline
  \citet{xu2023dynamics}& NN without bias & N/A & Square loss \newline normalized $l_2$ regularization& SGD & Global min & Each layer has rank $\leq 2$ \\ \hline
 \citet{galanti2022sgd}& NN without bias & N/A & Differentiable loss \newline $l_2$ regularization& SGD & Convergence of \newline NN weights & Each layer has rank $\leq Bm_l$ \\ \hline
  \textbf{This paper}& Two-layers NN with bias & N/A & Square loss \newline $l_2$ regularization& SGD & N/A & Each layer has rank $\leq 2$ \\ \hline
\end{tabular}
\end{table}

In this work, we study the low-rank bias of two-layer ReLU neural networks trained with mini-batch SGD and weight decay using square loss function. We demonstrate, both theoretically and empirically, that weight decay leads to the low-rank bias of ReLU neural networks. We then leverage the low-rank bias to obtain a smaller generalization error of the trained neural network model. Our contributions are summarized as follows: 
\begin{itemize}
    \item \textbf{Low-rank bias for sufficiently trained neural network.}
    In theorem \ref{thm:lowrankbias_perturbed}, we prove that the matrix weight parameters in the trained neural network is close to a rank two matrix and the distance between them depends on the batch size, the strength of the weight decay, and the batch gradient. In contrast to common assumptions on attaining global minimum and sufficient depth~\citep{timor2023implicit,jacot2023implicit}, strictly zero batch gradient~\citep{xu2023dynamics}, and convergence of weight matrices~\citep{galanti2022sgd}, we only assume that the neural network is sufficiently trained so that the batch gradient is small( see Assumption \ref{assump:minibatch_perturbed}). This result is demonstrated and supported by numerical evidence of both regression problem and classification problem in Section \ref{sec:numeric}.
    \item \textbf{Low-rank bias leads to better generalization.}
    In theorem \ref{thm:genbound}, we show that the low-rank bias leads to better generalization and thus implicitly explain the generalization effect of SGD with weight decay. For a two-layer ReLU neural networks with input dimension $m$ and width $n$ trained on $N$ data pairs, we prove that the generalization error bound can be improved from $\Ocal{\sqrt{\frac{mn \ln m \ln N}{N}}}$ to $\Ocal{\sqrt{\frac{(m+n) \ln m \ln N}{N}}}$. Our result thus provides an alternative explanation to the generalization power of SGD algorithm through the lens of low-rank bias. The improved generalization error bound is also echoed with the recent work of~\citet{park2022generalization}, where they leverage the local contraction property of SGD.
\end{itemize}

\subsection{Discussion of related works}

We have summarized relevant findings on low-rank bias in Table \ref{tab:relatedresults}.
The studies of \cite{ji2018gradient,ji2020directional,lyu2019gradient} investigated the directional convergence of the parameter $\theta$ for homogenized neural network in classification tasks. Their results demonstrated that $\theta$ converges to a direction that maximized the margin of the classifiers. A comparable result is presented in \citet{kumar2024early} for regression problems.
Our work aligns with these findings, showing that $\theta$ converges to a low-rank matrix.
For regression problems, \citet{timor2023implicit} established that weight decay induces a low-rank bias. However, our contribution differs from \citet{timor2023implicit} as their results are restricted to sufficiently deep neural networks.

To the best of the authors' knowledge, the most closely related works to our results are \citet{galanti2022sgd}, \citet{xu2023dynamics} and \citet{park2022generalization}.
\citet{galanti2022sgd} studies a fully connected neural network without bias vectors, trained with weight decay. They prove that when the norms of weight matrices converge, the normalized weight matrices are close to a matrix whose rank is bounded by the batch size and the number of patches in the input space. 
Our work differs from theirs in several key aspects: we do not assume the convergence of the neural network but instead consider the case where the neural network is sufficiently trained (cf. Assumption \ref{assump:minibatch_perturbed}). Additionally, we establish a stronger bound on the rank of weight matrices.

\citet{xu2023dynamics} extends the idea of normalized weight matrices in ~\citet{galanti2022sgd} and examines the mini-batch SGD training of neural networks with normalized WD. Rather than using the standard mini-batch SGD, they normalize all neural network weight matrices and train both the normalized weight matrices and the product of matrix norms together with mini-batch SGD. 
In contrast, our work employs the SGD algorithm with conventional WD, eliminating the need for an additional normalization step. Furthermore, ~\citet{xu2023dynamics} assumes that all batch gradients equal to zero, which is not practical in most regression problems. In comparison, we assume only that the batch gradients are small, and our assumption is numerically verified.

\citet{park2022generalization} does not address low-rank bias; instead, they leverage the piecewise contraction property of SGD optimization. 
They obtain a dimension-independent estimate of the generalization error due to this property. However, their proof relies on the assumption that the loss function is piecewise convex and smooth, and the generalization error estimate depends on the total number of pieces in the parameter space.
A straightforward calculation of the total number of pieces for a two-layer ReLU network leads to a generalization error of the order $\Ocal{\sqrt{\frac{ (\ln N + m)\ln N + \ln(1/\delta)}{N}}}$, which is essentially of the same order as our result when $m\geq n$.

\section{Low-rank bias of mini-batch SGD}
\subsection{Setup}
We consider a two-layer ReLU neural network $\phi_\theta:\Xcal \rightarrow \Rbb$
\begin{equation}\label{eqn:2layerNN}
    \phi_\theta(x)\coloneqq U \sigma(Vx+b) \,.
\end{equation}
Here $\Xcal \subset \Rbb^n$ is the input domain. The model parameter $\theta\coloneqq[U,V,b]$ consists of a row vector $U\in\Rbb^{1\times m}$, a matrix $V\in\Rbb^{m\times n}$, and a column vector $b\in\Rbb^{m \times 1}$. As the width $m$ of the neural network increases, the neural network can approximate any continuous function \cite{hornik1991approximation}.

Note that the ReLU activation function  $\sigma:\Rbb\rightarrow \Rbb$ is a piecewise linear function evaluated elementwisely on the vector $Vx+b$. For convenience, we rewrite the ReLU activation as a matrix operator $D(x,V,b)\in \{0,1\}^{m \times m}$, defined as follows:
\begin{equation}
    D(x,V,b) = \text{diag}(\text{sign}(\sigma(Vx+b))) \,.
\end{equation}
In other word, $D(x,V,b)$ is a diagonal matrix whose diagonal entries are either one or zero, depending on the sign of the first layer output $\sigma(Vx+b)$. Consequently, we can rewrite the two-layer NN as a product of three transformations
\begin{equation}\label{eqn:2layerNN2}
    \phi(x,\theta) = UD(x,V,b) (Vx+b) \,.
\end{equation}
We first show that $D(x,V,b)\in \{0,1\}^{m \times m}$ does not change its values for a small perturbation of $V$, consequently we have $\frac{\partial D(x,V,b)}{\partial V} = 0$ almost surely in the parameter space $\Theta$.
\begin{lemma}\label{lem:gradD}
    Consider a two-layer NN in \Eqref{eqn:2layerNN}, for any fixed $(x,b) \in \Rbb^n \times \Rbb^m$, 
    \[
    \frac{\partial D(x,V,b)}{\partial V} = 0\,, \quad \text{for all }V\in \Rbb^{m\times n}/ \Vcal^0 \,,
    \]
    where $\Vcal^0 \in \Rbb^{m\times n}$ is a measure zero set that depends on $x$ and $b$.
\end{lemma}

This result implies that the activation component of neural network model is not sensitive to small changes in the model parameter $V$, except on a measure zero set. 

We show below that such a property also affects the gradient of neural network model towards a low-rank bias.

\begin{lemma}\label{lem:rankone}
    Consider a two-layer NN in \Eqref{eqn:2layerNN} and fix $x\in \Rbb^n$ and $b\in \Rbb^m$,
    \begin{equation}
        \frac{\partial \phi(x;\theta)}{\partial V} = x U D(x,V,b) \quad \text{a.s. in }V\,. 
    \end{equation}
    That is, $\frac{\partial \phi(x;\theta)}{\partial V}$ is a rank one matrix a.s. in $V$.
\end{lemma}

This result implies that gradient-based training algorithms only update the weight matrix $V$ with rank one matrix increments. Although this does not directly imply that the algorithm converges to a low-rank matrix, we show below this would happen when combined with WD.

\subsection{Training dynamics of SGD with Weight Decay}
\label{sec:mini-batchSGD}

Given a training dataset $\Scal = \{z_i\coloneqq (x_i,y_i): \Rbb^n \times \Rbb\,, i  \in [N] \}$,
we consider the two-layer NN with following Mean Square Error (MSE) loss function with WD
\begin{equation}\label{eqn:loss}
    \Lcal_\Scal(\theta) = \frac{1}{2N} \sum_{i=1}^N|\phi(x_i,\theta) - y_i|^2 + \frac{\mu_U}{2}\|U\|_F^2 + \frac{\mu_V}{2}\|V\|_F^2 + \frac{\mu_b}{2}\|b\|^2 \,,
\end{equation}
where $\mu_U,\mu_V,\mu_b>0$ does not depend on $\theta$. 

In practice, the training is implemented with mini-batch SGD method. That is, in each iteration, the gradient is updated over batch $\Scal'\subset \Scal$ with batch size $|\Scal'| = B$,
\begin{equation}\label{eqn:mini-batchSGD}
    \frac{\partial \Lcal_{\Scal'}}{\partial V} = \frac{1}{B} \sum_{i\in \Scal'} (\phi(x_i,\theta) - y_i) \frac{\partial \phi(x_i,\theta)}{\partial V} + g(x_i,y_i) V \,.
\end{equation}
Note that we have chosen $\mu_V = \frac{1}{B} \sum_{i \in \Scal'} g(x_i,y_i)$ for some fixed positive function $g: \Rbb^n \times \Rbb \rightarrow \Rbb^+$. 

\subsubsection{low-rank bias of the critical point}
If the training converges to a limit $\theta^\ast$, then heuristically all batch gradient converges to zero. We first consider the following assumption.
\begin{assumption}\label{assump:minibatch}
    For a fixed batch size $2 \leq B < |\Scal|$, the batch gradient $\frac{\partial \Lcal_{\Scal'} }{\partial V}(\theta^\ast) = 0$ for all batches $\Scal'\subset \Scal$ of size $B$.
\end{assumption}

We can now prove the low-rank bias of training with weight decay.
\begin{theorem}\label{thm:lowrankbias}
    Consider a two-layer NN in \Eqref{eqn:2layerNN} trained with mini-batch SGD as in \Eqref{eqn:mini-batchSGD}, under Assumption \ref{assump:minibatch}, the neural network parameters converges to a matrix $V^\ast$ with rank at most two. In particular, if $g(\cdot,\cdot)=\mu_V$ is a constant, then $V^\ast$ is of rank one.
\end{theorem}

\subsubsection{low-rank bias in a neighborhood of critical points}
In practice, Assumption \ref{assump:minibatch} may not be realistic as it is impossible to train all batch gradients to zero. Instead, we consider the following assumption that the batch gradients are close to zero.
\begin{assumption}\label{assump:minibatch_perturbed}
    There exists $\varepsilon>0$ and a fixed batch size $2 \leq B < |\Scal|$, such that the batch gradient $\| \frac{\partial \Lcal_{\Scal'} }{\partial V}(\theta^\ast) \|_F \leq \varepsilon $ for all batches $\Scal'\subset \Scal$ of size $B$.
\end{assumption}
Under Assumption \ref{assump:minibatch_perturbed}, we show that the coefficient matrix $V^\ast$ is close to a low-rank matrix $\tilde{V}^\ast$.
\begin{theorem}\label{thm:lowrankbias_perturbed}
    Consider a two-layer NN in \Eqref{eqn:2layerNN} trained with mini-batch SGD as in \Eqref{eqn:mini-batchSGD}, under Assumption \ref{assump:minibatch_perturbed}, the neural network parameters $V^\ast$ satisfies that $\|V^\ast - \tilde{V}^\ast\|_F \leq C\varepsilon$ for some matrix $\tilde{V}^\ast$ with rank at most two. Here the constant $C$ only depends on the batch size $B$ and the choice of $g(\cdot,\cdot)$. In particular, if $g(\cdot,\cdot)=\mu_V$ is a constant, then $\tilde{V}^\ast$ is of rank one and $C = \frac{2B}{\mu_V}$.
\end{theorem}

\section{Generalization Error Analysis}
\subsection{Preliminaries}

In this section, we derive the generalization error of a low-rank neural network by following the techniques developed in \cite{anthony1999neural,bartlett2019nearly}.  
We first define a few notions that measures the complexity of a class of functions.
\begin{definition}[Covering Number]
    For any $\vep>0$ and a set $W$, we define the $l_1$ $\vep$-\textit{covering number} of $W$, denoted as $\Ncal(\vep,W,\|\cdot \|_1)$, to be minimal cardinality of a $\vep$-cover of $W$ under $l_1$ metric. 
\end{definition}
\begin{definition}[Uniform Covering Number]
    For any $\vep>0$, a class of functions $\Fcal$, and an integer $k$, we define the \textit{uniform covering number} $\Ncal_1(\vep,\Fcal,k)$ as
    \[
    \Ncal_1(\vep,\Fcal,k) = \max\{\Ncal(\vep, \Fcal|_\xsf, \|\cdot\|_1)\,, \xsf = [x_1,\ldots,x_k]\in \Xcal^k\}\,,
    \]
    where $ \Fcal|_\xsf\coloneqq \{(f(x_1),\ldots,f(x_k)): f\in \Fcal\}$.
\end{definition}
It can be shown that for a bounded function class $\Fcal$, the generalization error for any $f\in \Fcal$ can be uniformly bounded by the uniform covering number of $\Fcal$.
\begin{theorem}[Theorem 17.1 in \cite{anthony1999neural}]
\label{thm:gen_by_cn}
    Suppose $\Fcal $ is a set of functions defined on a domain $\Xcal$ and mapping into the real interval $[0,1]$. Let $z,z_1,\ldots,z_N$ be i.i.d. random variables on $\Zcal=\Xcal \times [0,1]$, $\vep$ any real number in $[0,1]$, and $N$ any positive integer. Then with probability at least $1-\delta$
    \[
    \sup_{f\in \Fcal} | \mathbb{E}_z\left[l(z,f)\right] - \frac{1}{N} \sum_{i=1}^N l(z_i,f) | \leq \vep \,,
    \]
    where $l(z,f) = |f(x)-y|^2$ is the square loss and the failure probability $\delta \leq 4 \Ncal_1(\vep/16,\Fcal,2N) \exp(-N\vep^2/32)$.
\end{theorem}
We then introduce the \textit{pseudo-dimension}, which is closely related to the uniform covering number.
\begin{definition}[Pseudo-dimension]
    Let $\Fcal$ be a class of functions from $\Xcal$ to $\Rbb$. The pseudo-dimension of $\Fcal$, denoted as $\Pdim(\Fcal)$, is the largest integer $m$ for which there exists $(x_1,\ldots,x_m,y_1,\ldots,y_m)\in \Xcal^m \times \Rbb^m$ such that for any $(b_1,\ldots,b_m)\in \{0,1\}^m$ there exists $f\in\Fcal$ such that
    \[
    \forall i, f(x_i) > y_i \Leftrightarrow b_i = 1 \,.
    \]
\end{definition}
\begin{theorem}[Theorem 18.4 in \cite{anthony1999neural}]
\label{thm:cn_by_pdim}
    Let $\Fcal$ be a nonempty set of real functions from $\Xcal$ to $[0,1]$, then
    \[
    \Ncal_1(\vep,\Fcal,k) \leq e(\Pdim(\Fcal)+1) \left(\frac{2e}{\vep}\right)^{\Pdim(\Fcal)} \,.
    \]
\end{theorem}
Combining the above results together, we obtain the following generalization bounds.
\begin{proposition}\label{prop:gen_by_pdim}
    Let $\Fcal$ be a nonempty set of real functions from $\Xcal$ to $[0,1]$, $z,z_1,\ldots,z_N$ be i.i.d. random variables on $\Zcal=\Xcal \times [0,1]$, $\vep$ any real number in $[0,1]$, and $N$ any positive integer. Then with probability at least $1-\delta$
    \[
    \sup_{f\in \Fcal} | \mathbb{E}_z\left[l(z,f)\right] - \frac{1}{N} \sum_{i=1}^N l(z_i,f) | \leq C \sqrt{\frac{\ln(1/\delta)}{N}} + C \sqrt{\frac{\Pdim(\Fcal)\ln N}{N}} \,,
    \]
    where $l(z,f) = |f(x)-y|^2$ is the $l_2$ loss and $C>0$ is an independent constant.
\end{proposition}

We now consider the class of two-layer ReLU neural network functions defined as follows:
\begin{equation}\label{eqn:NNclass}
\Fcal(m,n) \coloneqq \{ \phi_\theta(x)=U \sigma(Vx + b): \theta = [U,V,b] \in \Rbb^{1\times m}\times \Rbb^{m\times n}\times \Rbb^{m\times 1} \} \,.
\end{equation}
We further assume that the neural network functions are uniformly bounded.

\begin{assumption}\label{assump:boundedness}
    There exists an independent constant $L>0$ such that $\sup\limits_{f \in \Fcal(m,n),x\in\Xcal} |f(x)| \leq L$.
\end{assumption}

The pseudo-dimension of two-layer ReLU neural networks was studied in \cite{bartlett2019nearly}. 
We use the following simplified version of Theorem 7 in \cite{bartlett2019nearly}.

\begin{theorem}[Theorem 7 in \cite{bartlett2019nearly}]
    Let $\Fcal$ denotes neural networks with $W$ parameters and $U$ ReLU activation functions in $L$ layers, then
    \[
    \Pdim(\Fcal) \leq \Ocal{LW\ln(U)} \,.
    \]
\end{theorem}
A simple consequence for two-layer ReLU neural networks is
\begin{equation}\label{eqn:bartlett2019}
    \Pdim(\Fcal(m,n)) \leq C mn \ln(m) \,,
\end{equation}
where $C>0$ is an independent constant. As a consequence, an algorithm independent generalization bound can be obtained by combining \Eqref{eqn:bartlett2019} and Proposition \ref{prop:gen_by_pdim}.
\begin{corollary}
    Let $\phi_\theta$ be any two-layer neural network from \Eqref{eqn:NNclass}, $z,z_1,\ldots,z_N$ be i.i.d. random variables on $\Zcal=\Xcal \times [0,1]$. Under Assumption \ref{assump:boundedness}, then with probability at least $1-\delta$,
    \[
     | \mathbb{E}_z\left[l(z,\phi_\theta)\right] - \frac{1}{N} \sum_{i=1}^N l(z_i,\phi_\theta) | \leq C L^2 \sqrt{\frac{\ln(1/\delta)}{N}} + C L^2 \sqrt{\frac{mn \ln m\ln N}{N}} \,,
    \]
    where $l(z,\phi_\theta) = |\phi_\theta(x)-y|^2$ is the $l^2$ loss and $C>0$ is an independent constant.
\end{corollary}

The above estimate holds for any two-layer ReLU neural network functions thus it is an algorithm independent bound of the generalization error. In Theorem \ref{thm:lowrankbias}, we have shown that the neural network parameter $\theta$ is close to a low-rank parameter $\theta^\ast$ when trained with mini-batch SGD. This low-rank bias essentially implies the learned neural network function $\phi_{\theta^\ast}$ belongs to a neighborhood of a smaller function class than $\Fcal(m,n)$. We show below that low-rank bias can improve the generalization error from $\Ocal{\sqrt{\frac{mn \ln m \ln N}{N}}}$ to $\Ocal{\sqrt{\frac{(m+n) \ln m \ln N}{N}}}$.

\subsection{Generalization error of a low-rank NN}
For a two-layer ReLU neural network $\phi_\theta$ with a rank-$k$ matrix $V\in\Rbb^{m\times n}$, we reparametrize $V=V_2V_1$ with $V_1\in \Rbb^{k\times n}$ and $V_2\in \Rbb^{m\times k}$. In other words, we can rewrite $\phi_\theta(x)$ as a composition of a linear function $y = V_1 x$ and a two-layer ReLU neural network $\Tilde{\phi}_\theta(y) \in F(m,k)$ with input dimension $k$.
We thus define the linear function class
\[
L(k,n) = \{f:\Rbb^n \rightarrow \Rbb^k: f(x) = Ax, A \in \Rbb^{k\times n}\} \,.
\]
We define the composition of functions from $L(k,n)$ and $F(m,k)$, i.e. the low-rank two-layer ReLU neural networks, as the following
\begin{equation}\label{eqn:lowrankNNclass}
\tilde{\Fcal}(m,n,k) \coloneqq \{ f \circ g: f\in F(m,k), g\in L(k,n) \} \,.
\end{equation}
We can easily calculate the pseudo-dimension of $\tilde{\Fcal}(m,n,k)$ using the following property of pseudo-dimension.
\begin{lemma}\label{lem:pseudodim}
For any positive integers $m,n,k$,
\[
    \Pdim(\tilde{\Fcal}(m,n,k)) \leq C (m+n)k \ln(m) \,.
\]
\end{lemma}

We can now calculate the generalization error for neural networks trained with mini-batch SGD and WD.
\begin{theorem}\label{thm:genbound}
    Let $\phi_{\theta^\ast}$ be a two-layer neural network in \Eqref{eqn:2layerNN} trained with mini-batch SGD as in \Eqref{eqn:mini-batchSGD} with i.i.d. data $\zsf\in \Zcal^N$, and $z\in\Zcal$ is an independent copy of $z_1$.
    Suppose Assumption \ref{assump:boundedness} and Assumption \ref{assump:minibatch} is satisfied, then with probability at least $1-\delta$, we have
    \[
     | \mathbb{E}_z\left[l(z,\phi_{\theta^\ast})\right] - \frac{1}{N} \sum_{i=1}^N l(z_i,\phi_{\theta^\ast}) | \leq C L^2 \sqrt{\frac{\ln(1/\delta)}{N}} + C L^2 \sqrt{\frac{(m+n) \ln m\ln N}{N}} \,,
    \]
    where $l(z,\phi_\theta) = |\phi_\theta(x)-y|^2$ is the $l^2$ loss and $C>0$ is an independent constant.
\end{theorem}

\section{Numerical Experiments}\label{sec:numeric}
In this section, we numerically demonstrate that WD leads to low-rank bias for two-layer neural networks.
Our experiments are conducted on two datasets: the California housing dataset for regression tasks and the MNIST dataset for classification. We investigate how the regularization coefficient $\mu_{V}$ and batch size $B$ affect the rank of $V$, providing evidence to support Theorem \ref{thm:lowrankbias_perturbed}. Additionally, we validate Assumption \ref{assump:minibatch_perturbed}, which suggests that sufficient training leads to small batch gradients. The rank of the coefficient matrix $V$ is measured by the stable rank $\srankV = \frac{\|V\|^2_{F}}{\|V\|_2^2}$. The generalization error in \Eqref{eqn:decomp_error} is computed as the difference between the testing Mean Squared Error (MSE) and the training MSE.

\subsection{Experimental Setup}
\subsubsection{Training and Initialization}
All experiments are conducted using two-layer neural networks with network sizes and hyperparameters specified in Table \ref{tab:hyperparameters}. These networks are trained with mini-batch SGD and the loss function defined in \Eqref{eqn:loss}. Note that here we use the square loss instead of the cross entropy loss for classification task. 
The weights matrices $U$ and $V$ are initialized with Kaiming initialization~\citep{he2015delving}, while the bias $b$ is initialized from the zero vector. The learning rate is set to $0.0001$ with a decay rate of 0.95 every 200 epochs. For simplicity, we have chosen $g$ as a constant function thus $\mu_V$ remains constant during the training dynamics. 
\begin{table}[ht]
    \centering
    \caption{Hyperparameters for neural network training}
    \label{tab:hyperparameters}
    \begin{tabular}{lcccccc}
        \toprule
        \textbf{Dataset} & \textbf{Width of NN} & \boldmath$\mu_U$ & \boldmath$\mu_b$ & \textbf{Epoch} \\
        \midrule
        California Housing Prices & 8192   & 1e-4 & 1e-4 & 5000 \\
        MNIST              & 32768   & 1e-6 & 1e-6 & 2000 \\
        \bottomrule
    \end{tabular}
\end{table}
\subsubsection{Datasets}
\textbf{California Housing Prices}. 
The California housing dataset~\citep{pace1997sparse} comprises 20640 instances, each with eight numerical features. The target variable is the median house value for California districts, measured in hundreds of thousands of dollars. We fetch dataset by \texttt{scikit-learn} package~\citep{pedregosa2011scikit}, and randomly samples 1800 instances for training and 600 for testing.

\textbf{MNIST}. 
The MNIST dataset~\citep{deng2012mnist} consists of 70000 grayscale images of handwritten digits of size 28$\times$28. We use built-in dataset from \texttt{torchvision} package~\citep{torchvision} and randomly sampled 9000 images for training 1000 for testing. Pixel values were normalized from $[0,255]$ to $[-1, 1]$ via a linear transform, with labels remaining in the range $[0, 9]$.

\subsection{California Housing Price}
In Figure \ref{fig:CHP_stablerank_mu}, we plot the stable rank of $V$ for various values of $\mu_V$ and the singular values of $V$ when $\mu_V=10^{-4}$ and $\mu_V=1$. It is observed that the stable rank is close to one only for large weight decay when $\mu_V=0.1$ and $\mu_V=1$, while $V$ is almost full rank for small weight decay $\mu_V=10^{-4}$, $10^{-3}$ and $10^{-2}$. This supports our theoretical predictions in Theorem \ref{thm:lowrankbias} and Theorem \ref{thm:lowrankbias_perturbed} that larger $\mu_V$ shrinks the distance from a rank one matrix $V$, and thus leads to a low-rank bias. 
In Figure \ref{fig:CHP_training_mu}, we plot the stable rank $\srankV$, training MSE, and generalization error during the training. It again aligns with our theory that WD leads to the low-rank bias. Furthermore, although WD leads to larger training MSE, the best generalization error was achieved in the large WD regime $\mu_V = 0.05$. We also found the stable rank and generalization error have weak dependence on the batch size $B$; see more details in the Appendix \ref{sec:appendix_batchsize}.
To validate Assumption \ref{assump:minibatch_perturbed}, which suggests that sufficient training results in small batch gradient norms, 
we analyze all random batches of the final epoch. We compute the Frobenius norm of these random batch gradients and present the results as histograms in Figure \ref{fig:CHP_histogram_gradientnorm}. 
It is shown that the Frobenius norm of most batch gradients is smaller than 2, with the largest norm being less than 12, regardless of the values of $\mu_V$ (ranging from $10^{-4}$ to 1).
In comparison, the Frobenius norm of a random Gaussian matrix with zero mean and variance $0.01$ of the same size is approximately $26$.

\begin{figure}[t]
    \centering
    \begin{subfigure}[b]{0.4\textwidth}
        \centering
       
        \includegraphics[width=\textwidth]{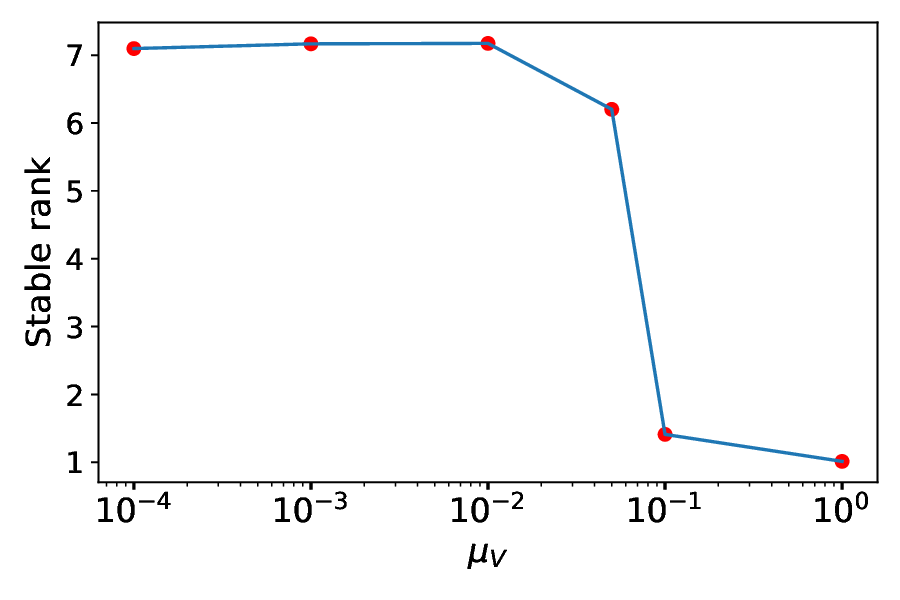}

    \end{subfigure}
    \hfill 
    \begin{subfigure}[b]{0.47\textwidth}
        \centering
      
        \includegraphics[width=\textwidth]{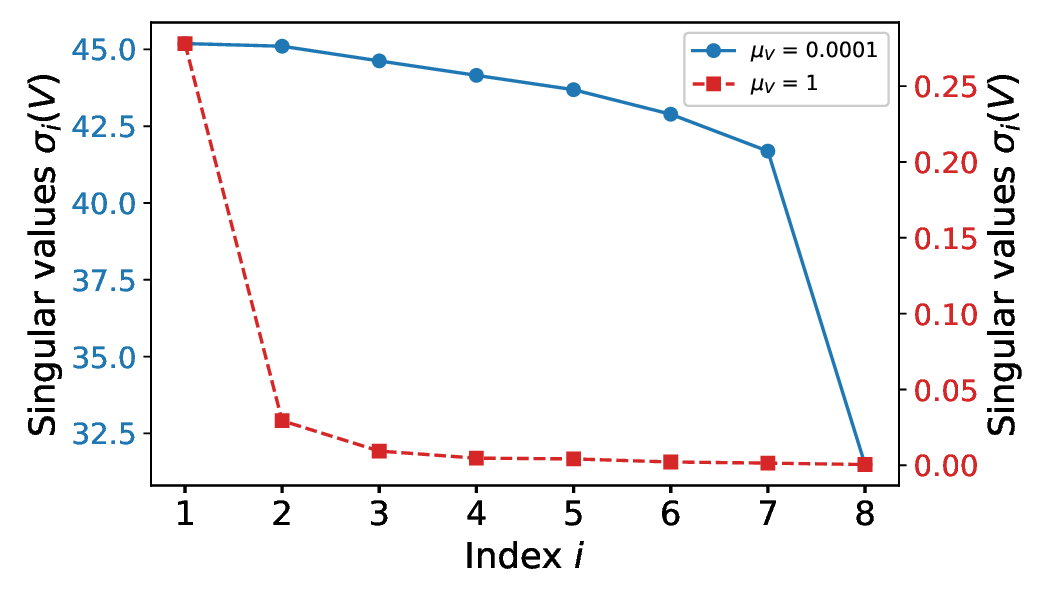}

    \end{subfigure}
    \caption{\textbf{California Housing Prices}.\  \textbf{Left: }Stable rank $\srankV$ versus $\mu_V$.\   \textbf{Right: }Singular values of $V$ for $\mu_V=0.0001$ and $1$. Here we fix batch size $B=16$.}
    \label{fig:CHP_stablerank_mu}
\end{figure}

\begin{figure}[htbp]
    \centering
    \begin{subfigure}[b]{0.32\textwidth}
        \centering
        \includegraphics[width=\textwidth]{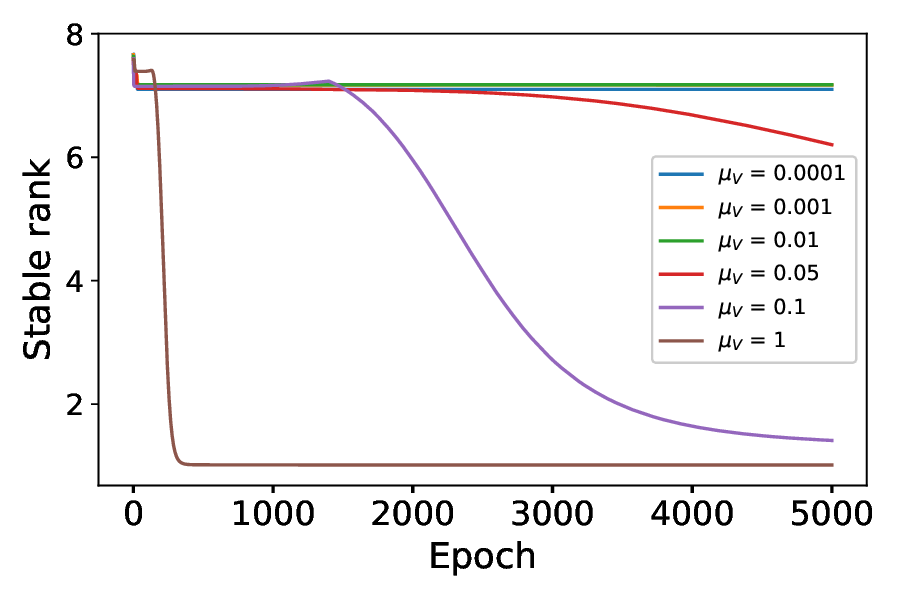}
    \end{subfigure}
    \hfill 
    \begin{subfigure}[b]{0.32\textwidth}
        \centering
        \includegraphics[width=\textwidth]{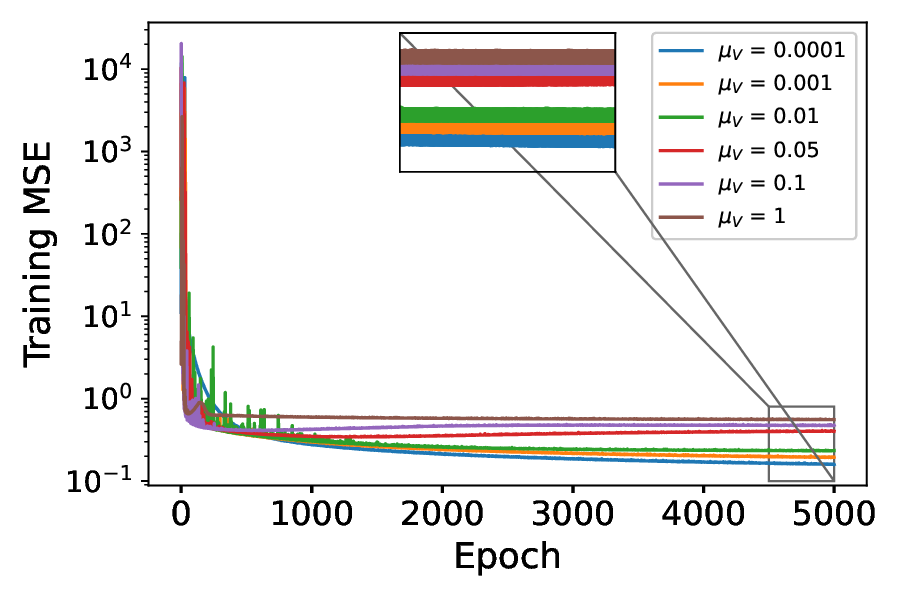}
    \end{subfigure}
    \hfill
    \begin{subfigure}[b]{0.32\textwidth}
        \centering
        \includegraphics[width=\textwidth]{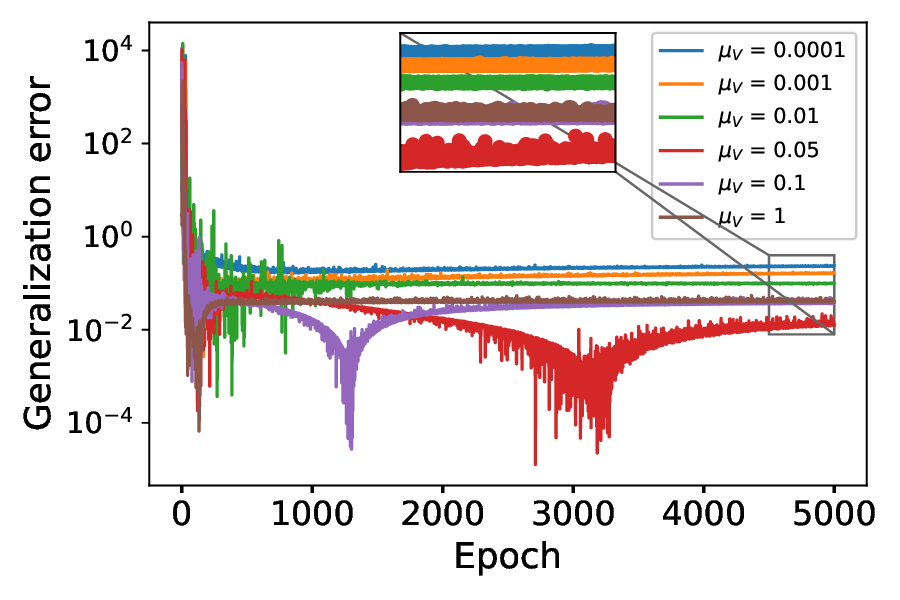}
    \end{subfigure}
    \caption{\textbf{California Housing Prices}.\ \textbf{Left:} Stable rank.\ \textbf{Middle: }Training MSE.\ \textbf{Right:} Absolute value of generalization error. Here we fix the batch size $B=16$. The sharp transition in the generalization error happens when it changes sign.}
    \label{fig:CHP_training_mu}
\end{figure}

\begin{figure}[htbp]
    \centering
   
    \begin{subfigure}[b]{0.32\textwidth}
        \centering
        \includegraphics[width = \textwidth]{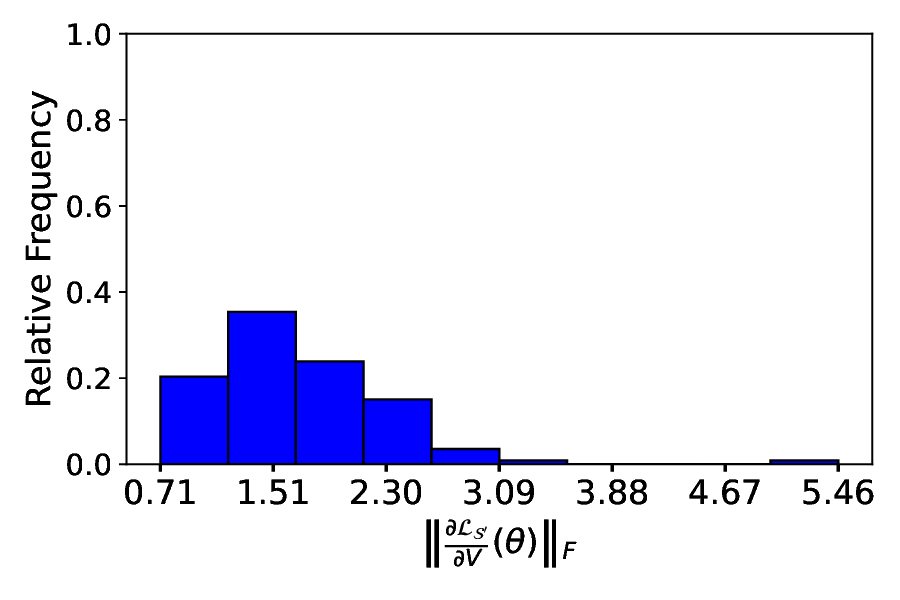}
        \caption{$\mu_{V}=0.0001$}
        
    \end{subfigure}
    \hfill
    \begin{subfigure}[b]{0.32\textwidth}
        \centering
        \includegraphics[width = \textwidth]{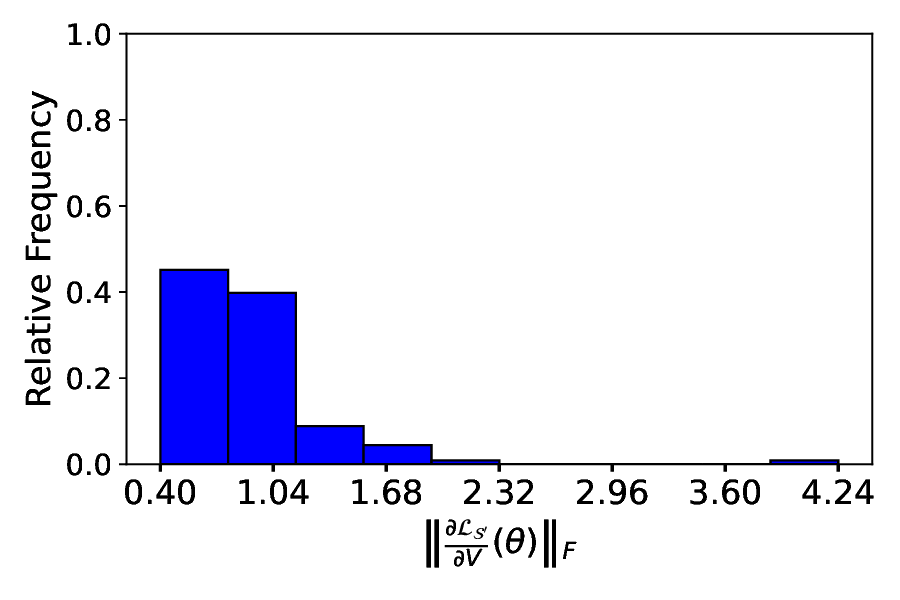}
        \caption{$\mu_{V}=0.001$}
        
    \end{subfigure}
    \hfill
    \begin{subfigure}[b]{0.32\textwidth}
        \centering
        \includegraphics[width = \textwidth]{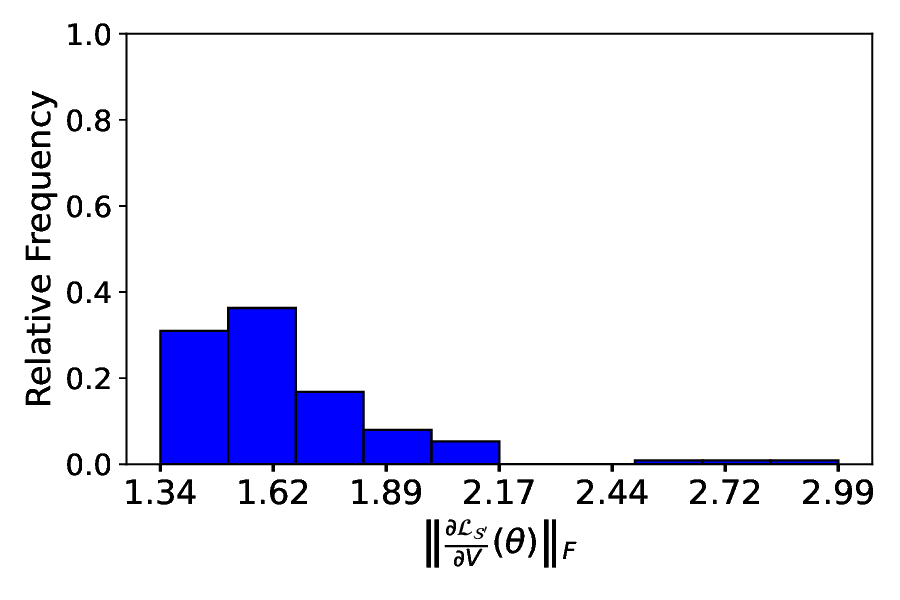}
        \caption{$\mu_{V}=0.01$}
        
    \end{subfigure}
    \\[1ex] 
    \begin{subfigure}[b]{0.32\textwidth}
        \centering
        \includegraphics[width = \textwidth]{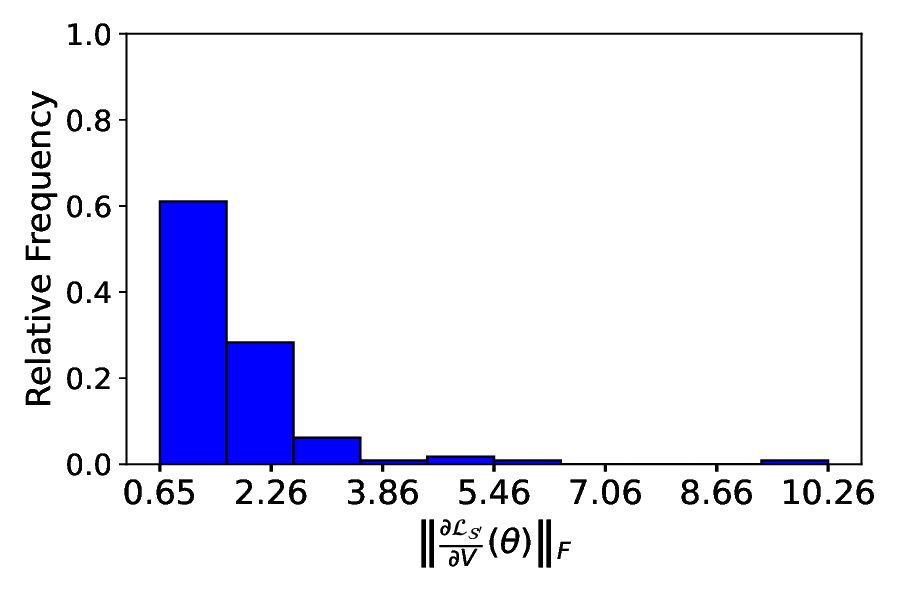}
        \caption{$\mu_{V}=0.05$}
        
    \end{subfigure}
    \hfill
    \begin{subfigure}[b]{0.32\textwidth}
        \centering
        \includegraphics[width = \textwidth]{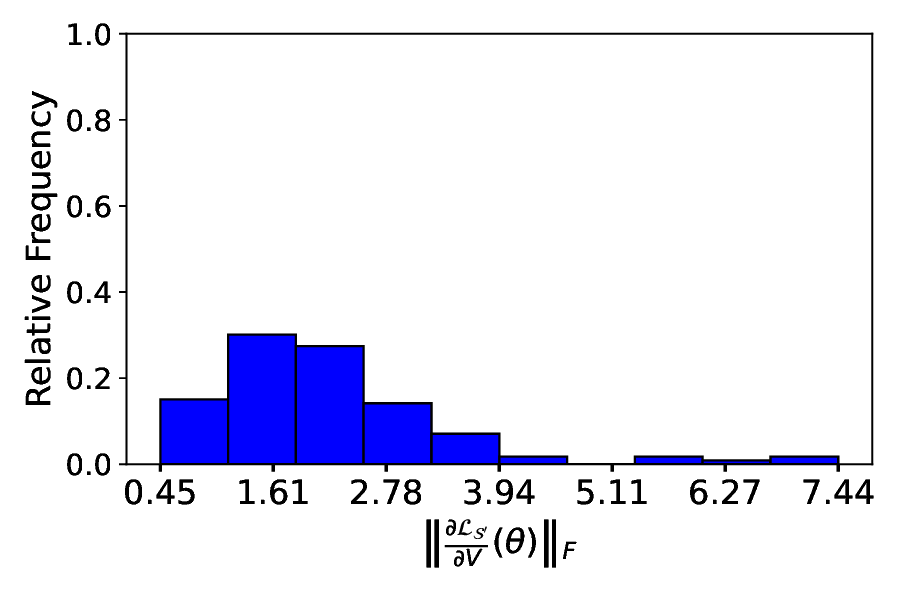}
        \caption{$\mu_{V}=0.1$}
        
    \end{subfigure}
    \hfill
    \begin{subfigure}[b]{0.32\textwidth}
        \centering
        \includegraphics[width = \textwidth]{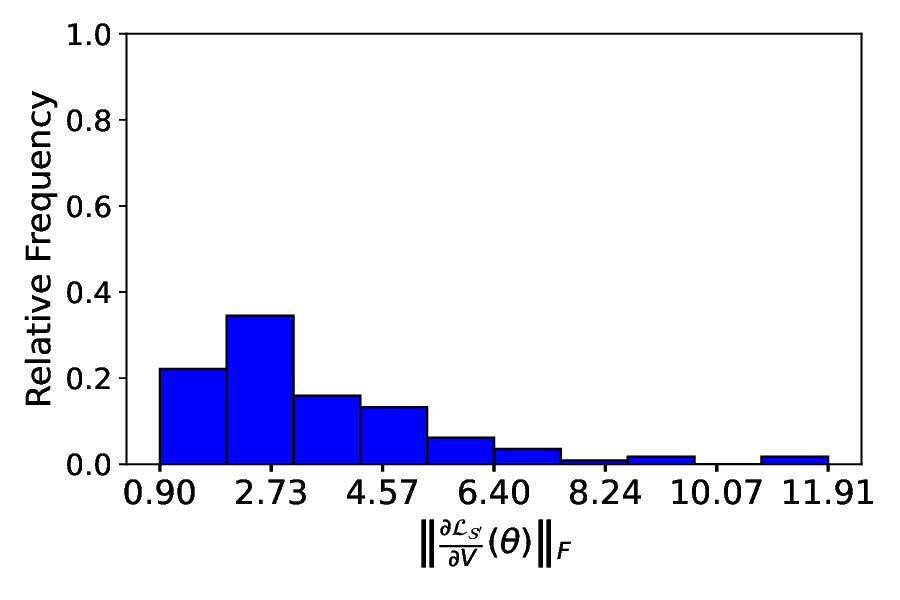}
        \caption{$\mu_{V}=1$}
        
    \end{subfigure}
    \caption{\textbf{California Housing Prices}. Histograms of the Frobenius norm of all batches in the final epoch. }
    \label{fig:CHP_histogram_gradientnorm}
\end{figure}

\subsection{MNIST Datasets}

To examine our theoretical results, we consider the square loss function thereby reformulating this classification problem as a regression problem.
In Figure \ref{fig:MNIST_stablerank_mu} we plot the stable rank $\srankV$ for various values of $\mu_V$ and the singular values of $V$ when $\mu_V=10^{-5}$ and $\mu_V=1$. In particular, the stable rank changes from $600$ to approximately $1$ when $\mu_V$ increases from $10^{-5}$ to $1$.
In Figure \ref{fig:MNIST_training_mu}, we plot the stable rank $\srankV$, training MSE and generalization error during training. It is observed that the best generalization error is achieved when $\mu_V=1$. 
In Figure \ref{fig:MNIST_accuracy}, we present training, testing, and generalization accuracy for different values of $\mu_V$ and batch size $B$. We observe that generalization accuracy tends to zero as we increase $\mu_V$, while it exhibits weak dependence on the batch size $B$.
Additionally, Figure \ref{fig:MNIST_histogram_gradientnorm} presents histograms of the Frobenius norm of batch gradients in the final epoch. The norm of most gradients are bounded by $50$ with the largest norm being $92$. In comparison, the Frobenius norm of a random Gaussian matrix with zero mean and $0.01$ variance of the same size is approximately $507$.
All these results are consistent with results from regression task. 

\begin{figure}[htbp]
    \centering
    \begin{subfigure}[b]{0.4\textwidth}
        \centering
        \includegraphics[width=\textwidth]{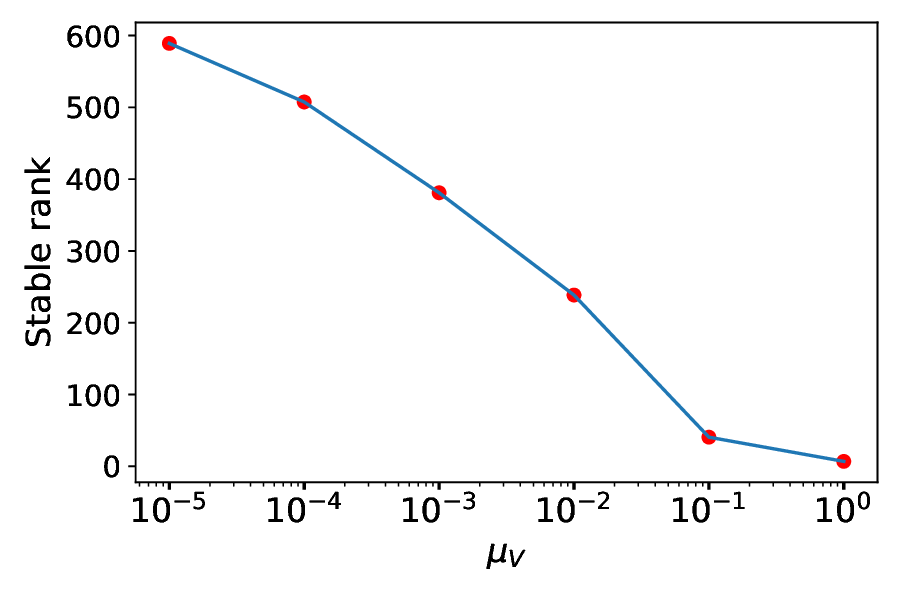}

    \end{subfigure}
    \hfill 
    \begin{subfigure}[b]{0.47\textwidth}
        \centering
        \includegraphics[width=\textwidth]{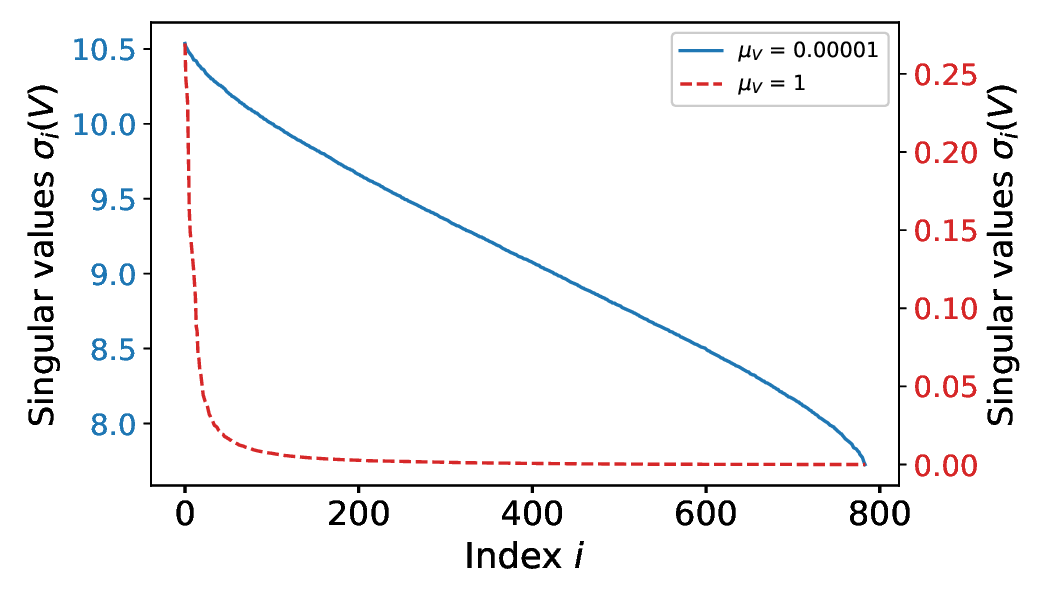}
    \end{subfigure}
    \caption{\textbf{MNIST}.\ \textbf{Left: }Stable rank versus $\mu_V$.\   \textbf{Right: }Singular values of $V$ for $\mu_V=10^{-5}$ and $1$. Here we fix batch size $B=64$.}
    \label{fig:MNIST_stablerank_mu}
\end{figure}

\begin{figure}[htbp]
    \centering
    \begin{subfigure}[b]{0.32\textwidth}
        \centering 
        \includegraphics[width=\textwidth]{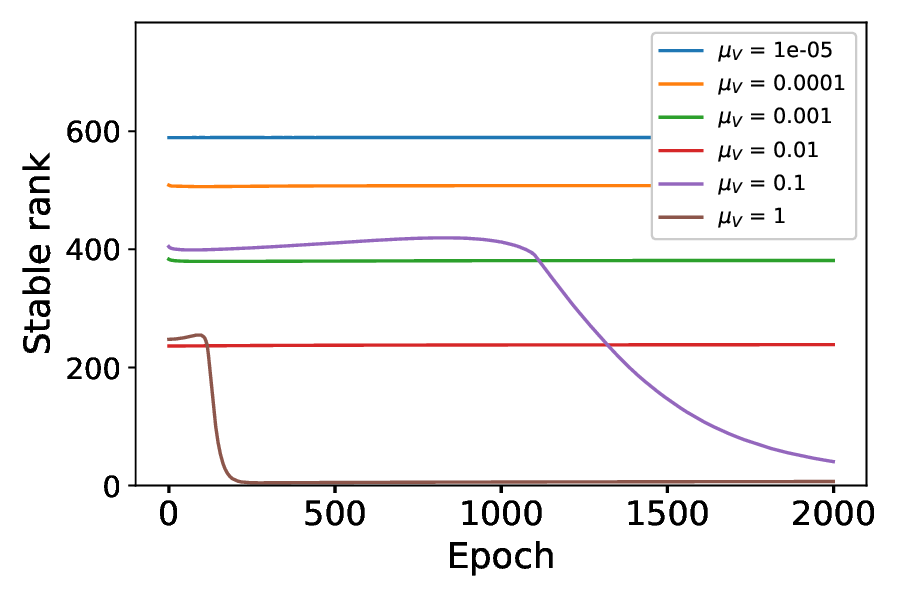}
    \end{subfigure}
    \hfill 
    \begin{subfigure}[b]{0.32\textwidth}
        \centering
        \includegraphics[width=\textwidth]{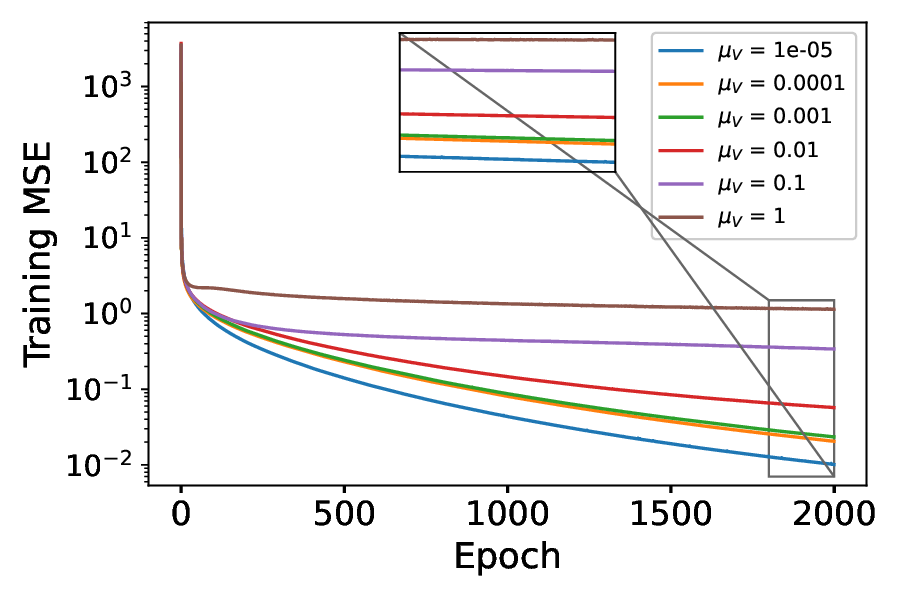}
    \end{subfigure}
    \hfill
    \begin{subfigure}[b]{0.32\textwidth}
        \centering
        \includegraphics[width=\textwidth]{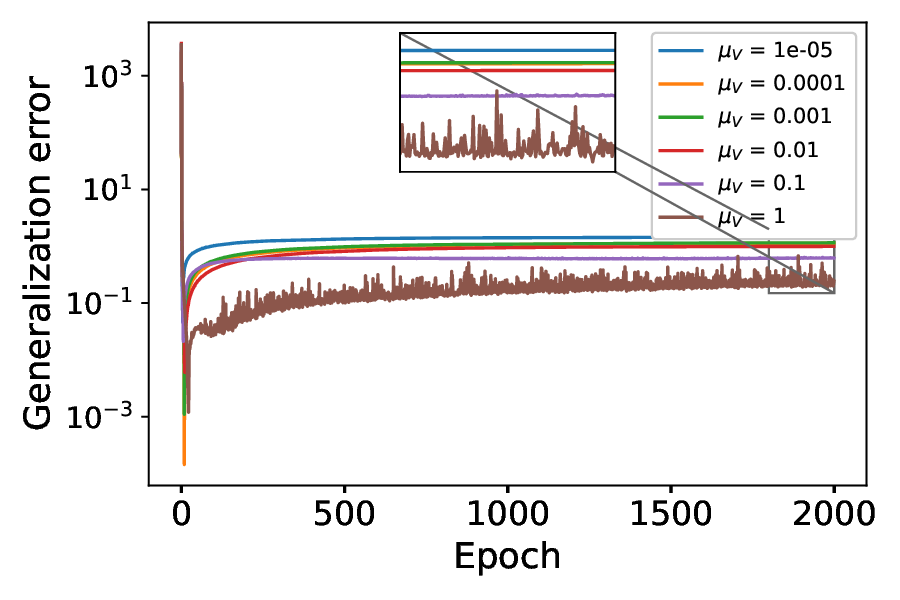}
    \end{subfigure}
    \caption{\textbf{MNIST}.\ \textbf{Left:} Stable rank.\ \textbf{Middle: }Training MSE.\ \textbf{Right:} Absolute value of generalization error. Here we fix the batch size $B=64$. The sharp transition in the generalization error happens when it changes sign.}
    \label{fig:MNIST_training_mu}
\end{figure}

\begin{figure}[ht]
    \centering
    \begin{subfigure}[htbp]{0.4\textwidth}
        \centering
   
        \includegraphics[width=\textwidth]{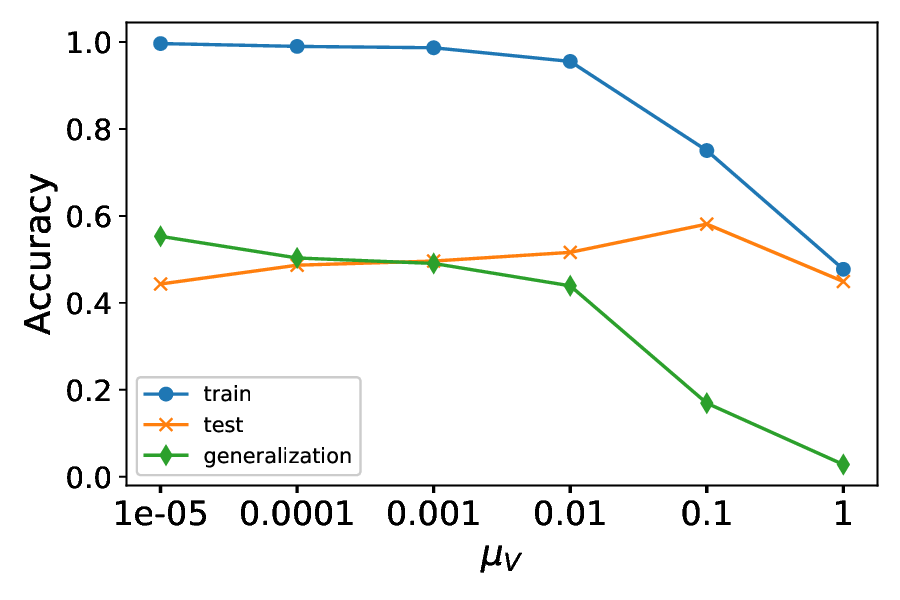}

    \end{subfigure}
    \begin{subfigure}[htbp]{0.4\textwidth}
        \centering
     
        \includegraphics[width=\textwidth]{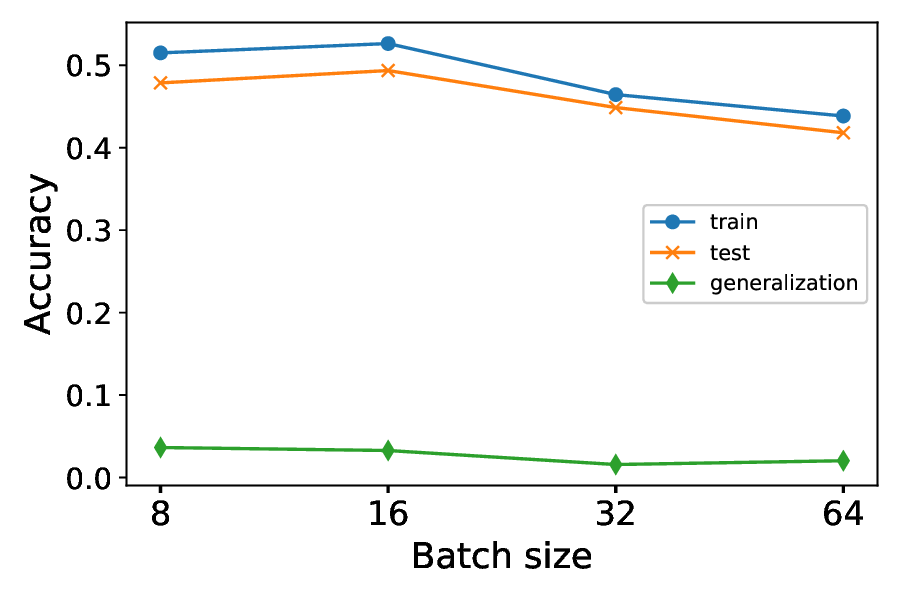}

    \end{subfigure}
    \caption{\textbf{MNIST}.\ \textbf{Left: } Accuracy when $B=64$.\  \textbf{Right: }Accuracy when $\mu_{V}=1$.}
    \label{fig:MNIST_accuracy}
\end{figure}

\begin{figure}[htbp]
    \centering
    \begin{subfigure}[b]{0.32\textwidth}
        \centering
        \includegraphics[width = \textwidth]
        {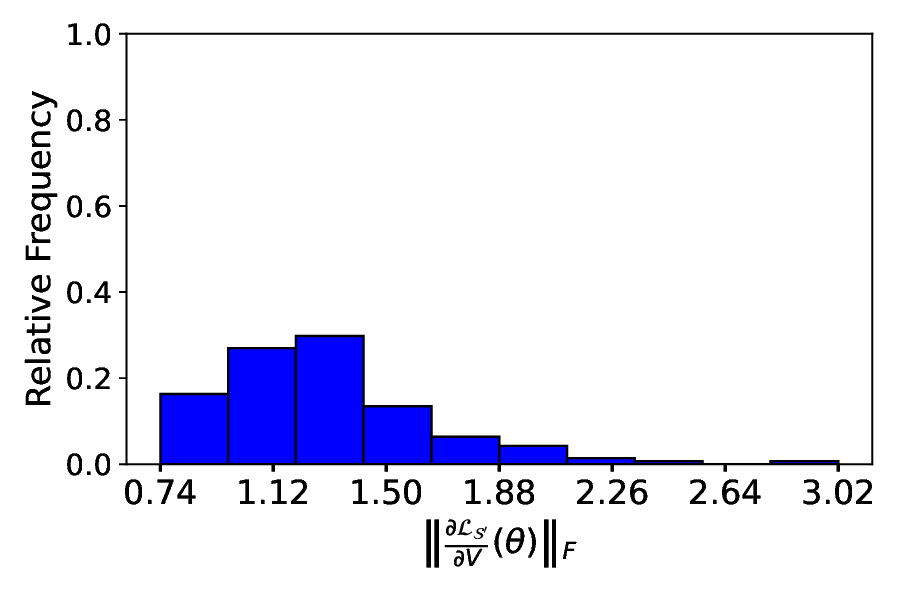}
        \caption{$\mu_{V}=0.0001$}
        \label{fig:sub1}
    \end{subfigure}
    \hfill
    \begin{subfigure}[b]{0.32\textwidth}
        \centering
        \includegraphics[width = \textwidth]{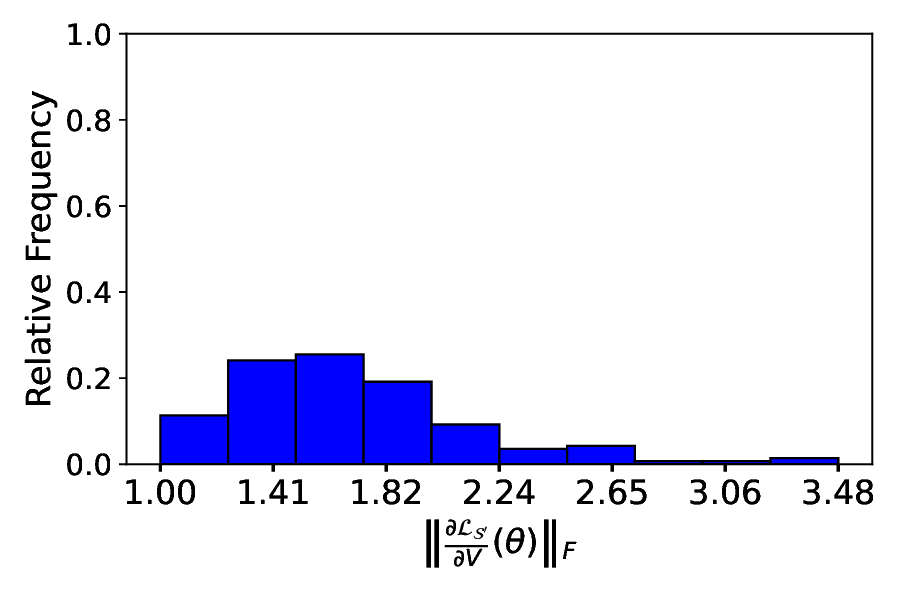}
        \caption{$\mu_{V}=0.001$}
        \label{fig:sub2}
    \end{subfigure}
    \hfill
    \begin{subfigure}[b]{0.32\textwidth}
        \centering
        \includegraphics[width = \textwidth]{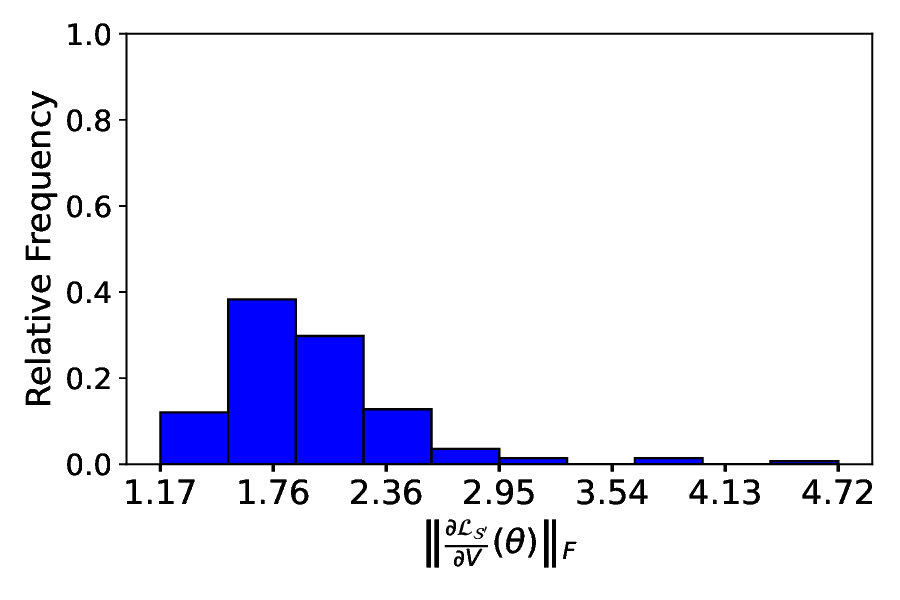}
        \caption{$\mu_{V}=0.01$}
        \label{fig:sub3}
    \end{subfigure}
    \\[1ex] 
    \begin{subfigure}[b]{0.32\textwidth}
        \centering
        \includegraphics[width = \textwidth]{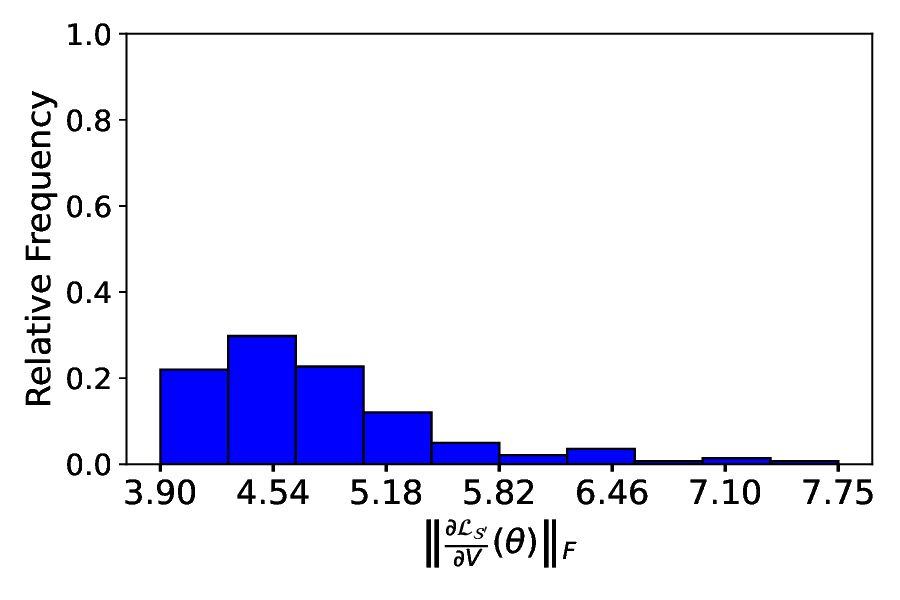}
        \caption{$\mu_{V}=0.05$}
        \label{fig:sub4}
    \end{subfigure}
    \hfill
    \begin{subfigure}[b]{0.32\textwidth}
        \centering
        \includegraphics[width = \textwidth]{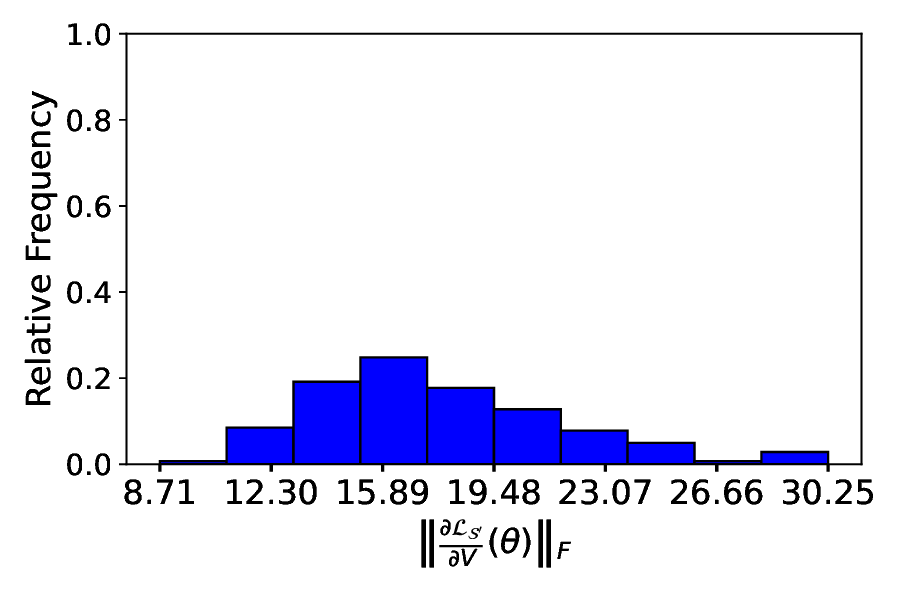}
        \caption{$\mu_{V}=0.1$}
        \label{fig:sub5}
    \end{subfigure}
    \hfill
    \begin{subfigure}[b]{0.32\textwidth}
        \centering
        \includegraphics[width = \textwidth]{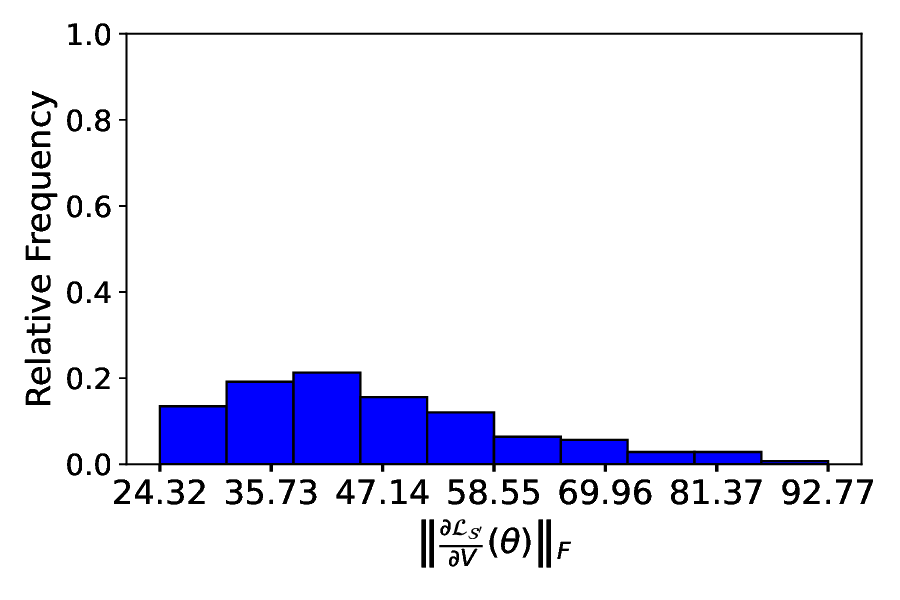}
        \caption{$\mu_{V}=1$}
        \label{fig:sub6}
    \end{subfigure}
    \caption{\textbf{MNIST}. Histograms of $\| \frac{\partial \Lcal_{\Scal'} }{\partial V}(\theta) \|_F$, where $\Scal'$ is a batch sampled from training dataset $\Scal$ and $\theta$ depends on $\mu_V$. }
    \label{fig:MNIST_histogram_gradientnorm}
\end{figure}

\section{Discussion}
Bounding generalization error is a fundamental problem in learning theory. Investigating the implicit bias of different algorithms helps address this problem by offering algorithm-specific estimates~\citep{neyshabur2017exploring}. Among these biases, the low-rank bias is of particular interest due to its potential for better compression and generalization  in many classification tasks. In this paper, we provide a mathematical proof that weight decay induces a low-rank bias in two-layer neural networks without relying on unrealistic assumptions. Our work explicitly shows how the gradient structure of neural networks, combined with weight decay, results in a low-rank weight matrix. We then further demonstrate that low-rank bias leads to better generalization bound. 
An interesting follow-up question is whether our analysis can be extended to deeper ReLU neural networks and other variants used in computer vision tasks, such as convolutional neural networks.
Another natural extension of our results involves adaptive regularized SGD as in~\Eqref{eqn:mini-batchSGD}, where the regularization strength varies across different batches.
In the proof of Theorem \ref{thm:lowrankbias_perturbed}, we show that the distance between trained weight matrix $V^\ast$ and a rank-two matrix $\tilde{V}^\ast$ is bounded by $\frac{2B }{ |g(x_{i_1},y_{i_1}) -  g(x_{i_2},y_{i_2})| }$. This suggests that a smaller generalization error bound can be achieved if the function $g$ can distinguish more effectively between different data points.

\bibliographystyle{plainnat}
\bibliography{templateArxiv}

\appendix
\section{Appendix}
\subsection{Proofs of Lemmas, propositions and theorems}
\begin{proof}[Proof of Lemma \ref{lem:gradD}]
    Denote $D_{ii}$ as the $i$-th diagonal entry of $D$ for $i = 1,\ldots,m$, then we have
    \[
    D_{ii}(x,V,b) = 1 \Leftrightarrow v_i^\top x + b_i >0 \,, \quad \text{and} \quad D_{ii}(x,V,b) = 0 \Leftrightarrow v_i^\top x + b_i \leq 0 \,,
    \]
    where $v_i^\top\in\Rbb^n$ denotes the $i$-th row of $V$. For a fixed $\hat{v}_i$ such that $\hat{v}_i^\top x + b >0$, one can construct a neighborhood $\Vcal_i$ around $\hat{v}_i$ such that $v_i^\top x + b_i >0$ for all $v_i \in \Vcal_i$. This implies that $D_{ii}(x,V,b)=1$ on $\Vcal_i$ and thus $\frac{\partial D(x,V,b)}{\partial v_i} = 0$ at $\hat{v}_i$. Analogously, $\frac{\partial D(x,V,b)}{\partial v_i} = 0$ at $\hat{v}_i$ as long as $\hat{v}_i^\top x + b <0$. Therefore, we have
    \[
    \frac{\partial D(x,V,b)}{\partial v_i} = 0
    \]
    for any $v_i$ except a measure zero set $\Vcal^0_i\coloneqq \{v_i \in \Rbb^n: v_i^\top x + b_i  = 0\}$. We thus conclude that $\frac{\partial D(x,V,b)}{\partial V} = 0$ for all $V$ except on $\Vcal^0 \coloneqq \{V\in \Rbb^{m \times n}: v_i \in \Vcal^0_i\,,\text{for some }i\in[m]\}$, which is a measure zero set.
\end{proof}

\begin{proof}[Proof of Lemma \ref{lem:rankone}]
    Apply product rule to \Eqref{eqn:2layerNN2}, we have
    \[
    \frac{\partial \phi(x;\theta)}{\partial V} = U \left(\frac{\partial D(x,V,b)}{\partial V}\right) (Vx+b) + x U D(x,V,b) \,.
    \]
    The conclusion holds by applying Lemma \ref{lem:gradD} and noting $xU \in \Rbb^{n \times m}$ is a rank one matrix.
\end{proof}

\begin{proof}[Proof of Theorem \ref{thm:lowrankbias}]
    We consider two batches $\Scal_1',\Scal_2' \subset \Scal$ with $|\Scal_1'| = |\Scal_1'| = B$ such that they only differ on one data pair, meaning that $\Scal_1'/\Scal_2' = \{(x_{i_1},y_{i_1})\}$ and $\Scal_2'/\Scal_1' = \{(x_{i_2},y_{i_2})\}$ for some $i_1 \neq i_2 \in \Scal$.
    Then under Assumption \ref{assump:minibatch}, we must have
    \begin{equation}\label{eqn:KKT}
    \frac{1}{B} \sum_{i\in \Scal_j'} (\phi(x_i,\theta^\ast) - y_i) \frac{\partial \phi(x_i,\theta^\ast)}{\partial V} + g(x_i,y_i) V^\ast = 0\,,\quad j = 1,2\,.
    \end{equation}
    Consequently, we have
    \begin{equation}\label{eqn:minibatch_diff}
    (\phi(x_{i_1},\theta^\ast) - y_{i_1}) \frac{\partial \phi(x_{i_1},\theta^\ast)}{\partial V} + g(x_{i_1},y_{i_1}) V^\ast = (\phi(x_{i_2},\theta^\ast) - y_{i_2}) \frac{\partial \phi(x_{i_2},\theta^\ast)}{\partial V} + g(x_{i_2},y_{i_2}) V^\ast  \,.
    \end{equation}
    We further choose $i_1$ and $i_2$ such that $g(x_{i_1},y_{i_1}) \neq g(x_{i_2},y_{i_2})$. Then we can solve for $V^\ast$,
    \[
    V^\ast = -\frac{(\phi(x_{i_1},\theta^\ast) - y_{i_1}) \frac{\partial \phi(x_{i_1},\theta^\ast)}{\partial V} - (\phi(x_{i_2},\theta^\ast) - y_{i_2}) \frac{\partial \phi(x_{i_2},\theta^\ast)}{\partial V}}{g(x_{i_1},y_{i_1}) - g(x_{i_2},y_{i_2})} \,.
    \]
    We apply Lemma \ref{lem:rankone} and conclude that $V^\ast$ is at most rank two if $V^\ast$ is not in the measure zero set. 

    If our choice of data pairs $(x_{i_1},y_{i_1})$ and $(x_{i_2},y_{i_2})$ do not exist. This implies that $g(x_i,y_i) = c>0$ is identical for all $(x_i,y_i)\in \Scal$. Since \Eqref{eqn:minibatch_diff} still holds for our choice of $\Scal_1'$ and $\Scal_2'$ we have
    \[
    (\phi(x_{i_1},\theta^\ast) - y_{i_1}) \frac{\partial \phi(x_{i_1},\theta^\ast)}{\partial V} = (\phi(x_{i_2},\theta^\ast) - y_{i_2}) \frac{\partial \phi(x_{i_2},\theta^\ast)}{\partial V} \,.
    \]
    Since $\Scal_1'$ and $\Scal_2'$ are arbitrary, and so are $i_1$ and $i_2$, we conclude that
    \[
    (\phi(x_{i_1},\theta^\ast) - y_{i_1}) \frac{\partial \phi(x_{i_1},\theta^\ast)}{\partial V} = R
    \]
    for some rank one matrix $R$.
    Then \Eqref{eqn:KKT} implies that $V^\ast = -\frac{1}{c}R$ is a rank one matrix. 
\end{proof}

\begin{proof}[Proof of Theorem \ref{thm:lowrankbias_perturbed}]
    The proof is similar to that of Theorem \ref{thm:lowrankbias}. We choose two $\Scal'_1,\Scal'_2 \subset \Scal$ the same way. Then under Assumption \ref{assump:minibatch_perturbed}, we must have
    \begin{equation}\label{eqn:KKT_perturbed}
        \| \frac{1}{B} \sum_{i\in \Scal_j'} (\phi(x_i,\theta^\ast) - y_i) \frac{\partial \phi(x_i,\theta^\ast)}{\partial V} + g(x_i,y_i) V^\ast \|_F \leq \varepsilon\,,\quad j = 1,2\,.
    \end{equation}
    Consequently, we have
    \begin{equation}\label{eqn:minibatch_diff_perturbed}
    \| (\phi(x_{i_1},\theta^\ast) - y_{i_1}) \frac{\partial \phi(x_{i_1},\theta^\ast)}{\partial V} + g(x_{i_1},y_{i_1}) V^\ast - \left( (\phi(x_{i_2},\theta^\ast) - y_{i_2}) \frac{\partial \phi(x_{i_2},\theta^\ast)}{\partial V} + g(x_{i_2},y_{i_2}) V^\ast \right) \|_F \leq 2B \varepsilon \,.
    \end{equation}
    If there exists $i_1$ and $i_2$ such that $g(x_{i_1},y_{i_1}) \neq g(x_{i_2},y_{i_2})$, then we can derive that
    \[
    \| V^\ast - \tilde{V}^\ast\|_F \leq C \varepsilon \,, \text{ where  } C = \frac{2B }{ |g(x_{i_1},y_{i_1}) -  g(x_{i_2},y_{i_2})| } \,,
    \]
    Here $\tilde{V}^\ast \coloneqq \frac{(\phi(x_{i_1},\theta^\ast) - y_{i_1}) \frac{\partial \phi(x_{i_1},\theta^\ast)}{\partial V} - (\phi(x_{i_2},\theta^\ast) - y_{i_2}) \frac{\partial \phi(x_{i_2},\theta^\ast)}{\partial V}}{g(x_{i_1},y_{i_1}) - g(x_{i_2},y_{i_2})} $ is a matrix with rank at most two due to Lemma \ref{lem:rankone}. 
    
    If for any $i_1$ and $i_2$, $g(x_{i_1},y_{i_1}) = g(x_{i_2},y_{i_2})$, then $g(x_i,y_i)=c$ for all $(x_i,y_i) \in \Scal$. 
    We first notice that \Eqref{eqn:KKT_perturbed} can be simplified to the following.
    \begin{equation}\label{eqn:KKT_perturbed2}
     \| c V^\ast +\frac{1}{B} \sum_{i\in \Scal'} (\phi(x_i,\theta^\ast) - y_i) \frac{\partial \phi(x_i,\theta^\ast)}{\partial V}   \|_F \leq \varepsilon\,,\quad \text{for any } \Scal'\subset \Scal \text{ with }|\Scal'|= B\,.
    \end{equation}
    From \Eqref{eqn:minibatch_diff_perturbed}, we can derive that
    \[
    \|  (\phi(x_{i_1},\theta^\ast) - y_{i_1}) \frac{\partial \phi(x_{i_1},\theta^\ast)}{\partial V}  - \left( (\phi(x_{i_2},\theta^\ast) - y_{i_2}) \frac{\partial \phi(x_{i_2},\theta^\ast)}{\partial V} \right) \|_F \leq 2 B\varepsilon \,.
    \]
    Without loss of generality, we may define $\tilde{V}^\ast \coloneqq  (\phi(x_{i_1},\theta^\ast) - y_{i_1}) \frac{\partial \phi(x_{i_1},\theta^\ast)}{\partial V} $ for any $i_1$, then the above equation implies that $(\phi(x_{i},\theta^\ast) - y_{i}) \frac{\partial \phi(x_{i},\theta^\ast)}{\partial V} $ is within an $2B\varepsilon$-ball of $\tilde{V}^\ast$ for all $i$, so is the average for any minibatch $\Scal'\subset \Scal$ with $|\Scal| = B$. In other words,
    \[
    \| \frac{1}{B} \sum_{i\in \Scal'} (\phi(x_i,\theta^\ast) - y_i) \frac{\partial \phi(x_i,\theta^\ast)}{\partial V}   - \tilde{V}^\ast\|_F \leq 2B\varepsilon \,.
    \]
    Combined this equation with \Eqref{eqn:KKT_perturbed2}, we apply triangle inequality and derive that
    \[
    \| c V^\ast + \tilde{V}^\ast \| \leq (2B+1)\varepsilon \,,
    \]
    which further implies that 
    \[
    \| V^\ast - (-\frac{1}{c}\tilde{V}^\ast) \|_F \leq C\varepsilon \,,
    \]
    where $C =  \frac{2B}{c}>0$ and $(-\frac{1}{c}\tilde{V}^\ast)$ is a rank one matrix.
    
\end{proof}

\begin{proof}[Proof of Proposition \ref{prop:gen_by_pdim}]
    Apply Theorem \ref{thm:gen_by_cn}, we choose $\delta = 4 \Ncal_1(\vep/16,\Fcal,2N) \exp(-N\vep^2/32)$, which implies that
    \[
    \sqrt{\frac{32 \ln(4/\delta)}{N}} \leq \vep \leq \sqrt{\frac{32 \ln(4/\delta)}{N}} + \sqrt{\frac{32 \ln \Ncal_1(\vep/16,\Fcal,2N)}{N}} \,.
    \]
    Apply Theorem \ref{thm:cn_by_pdim}, we have further that
    \[
    \vep \leq C \sqrt{\frac{ \ln(1/\delta)}{N}} + C \sqrt{\frac{\ln \Pdim(\Fcal) + \Pdim(\Fcal) \ln(1/\vep)}{N}} \,.
    \]
    To eliminate the dependence on $\vep$ of second term on the RHS, we plug in $\sqrt{\frac{32 \ln(4/\delta)}{N}} \leq \vep$ and obtain that
    \[
     \vep \leq C \sqrt{\frac{ \ln(1/\delta)}{N}} + C \sqrt{\frac{\ln \Pdim(\Fcal) + \Pdim(\Fcal) \ln(N)}{N}} \,,
    \]
    which completes the proof.
\end{proof}

\begin{proof}[Proof of Lemma \ref{lem:pseudodim}]
    We can treat any $f\in \tilde{\Fcal}(m,n,k)$ as a three layer neural network with width $m$ and the total number of parameters is $nk+mk+m=\Ocal{(m+n)k}$. The conclusion can be obtained by applying Theorem 7 in \cite{bartlett2019nearly}.
\end{proof}

\begin{proof}[Proof of Theorem \ref{thm:genbound}]
    This is a direct consequence of Lemma \ref{lem:pseudodim}, Proposition \ref{prop:gen_by_pdim} and Theorem \ref{thm:lowrankbias}.
\end{proof}

\subsection{Numerical Experiments with different batch size}
\label{sec:appendix_batchsize}

In Figure \ref{fig:CHP_stablerank_bs}, we plot the stable rank $\srankV$ and generalization error for various batch size $B$ in the strong weight decay regime $\mu_V=1$. 
Despite small variations, the stable rank and generalization error remain approximately constant across different batch sizes. Similar results are observed for MNIST datasets in Figure \ref{fig:MNIST_stablerank_batch}.
\begin{figure}[htbp]
    \centering
    \begin{subfigure}[b]{0.4\textwidth}
        \centering
        \includegraphics[width=\textwidth]{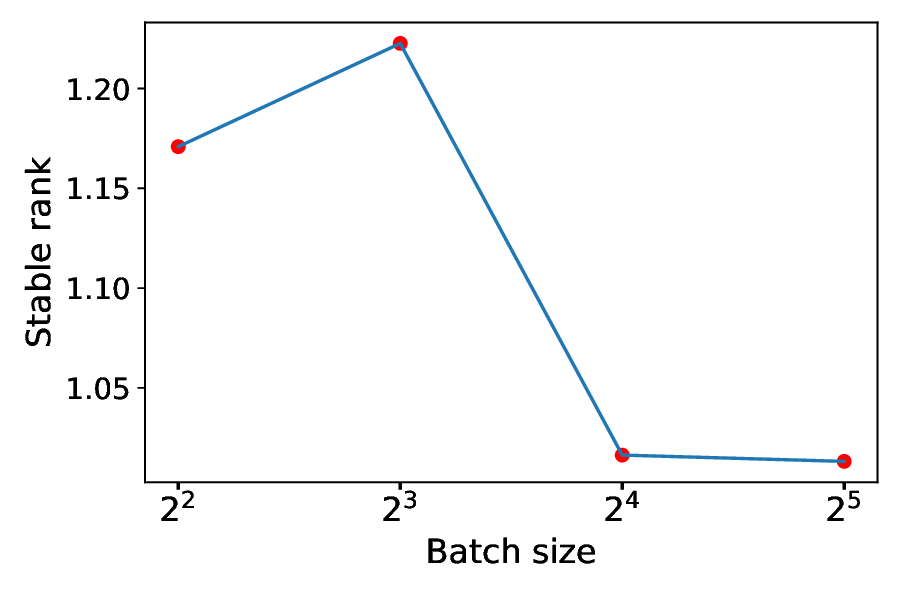}
    \end{subfigure}
    \hfill  
    \begin{subfigure}[b]{0.4\textwidth}
        \centering
         \includegraphics[width=\textwidth]{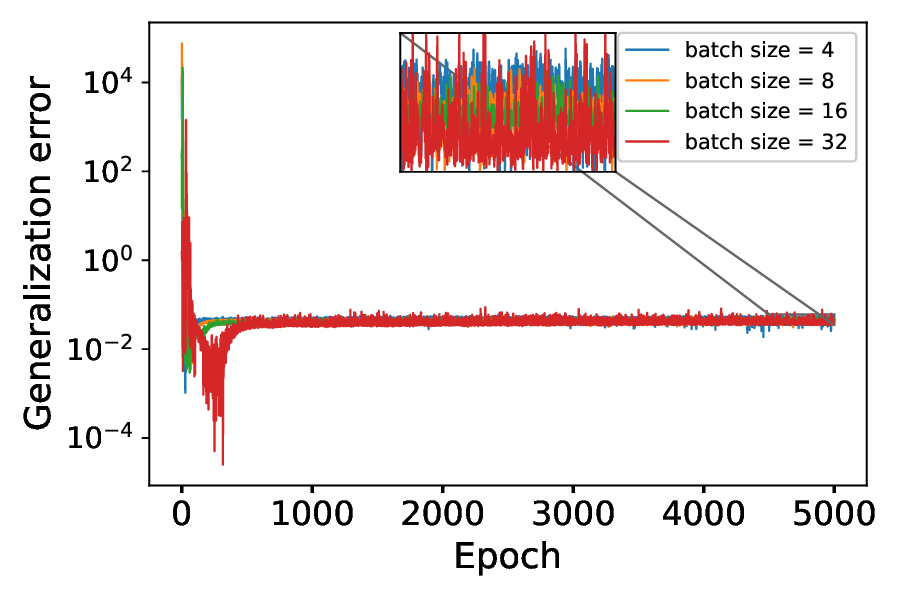}
   
    \end{subfigure}
    \caption{\textbf{California Housing Prices.}\ \textbf{Left: }Stable rank $\srankV$ versus batch size.\   \textbf{Right: }Absolute value of generalization error. Here we fix $\mu_{V}=1$.}
    \label{fig:CHP_stablerank_bs}
\end{figure}

\begin{figure}[htbp]
    \centering
    \begin{subfigure}[htbp]{0.4\textwidth}
        \centering
        \includegraphics[width=\textwidth]{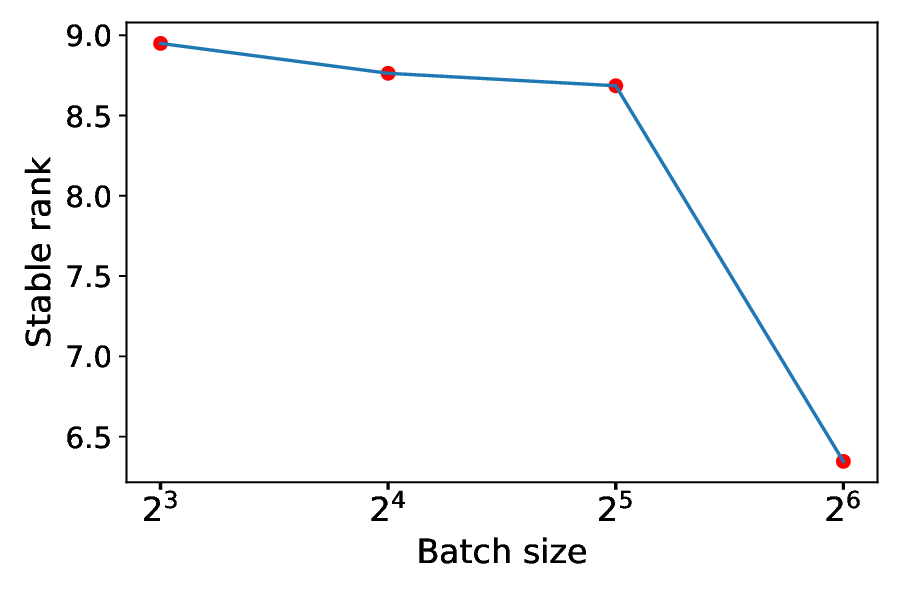}
    \end{subfigure}
    \begin{subfigure}[htbp]{0.4\textwidth}
        \centering
        \includegraphics[width=\textwidth]{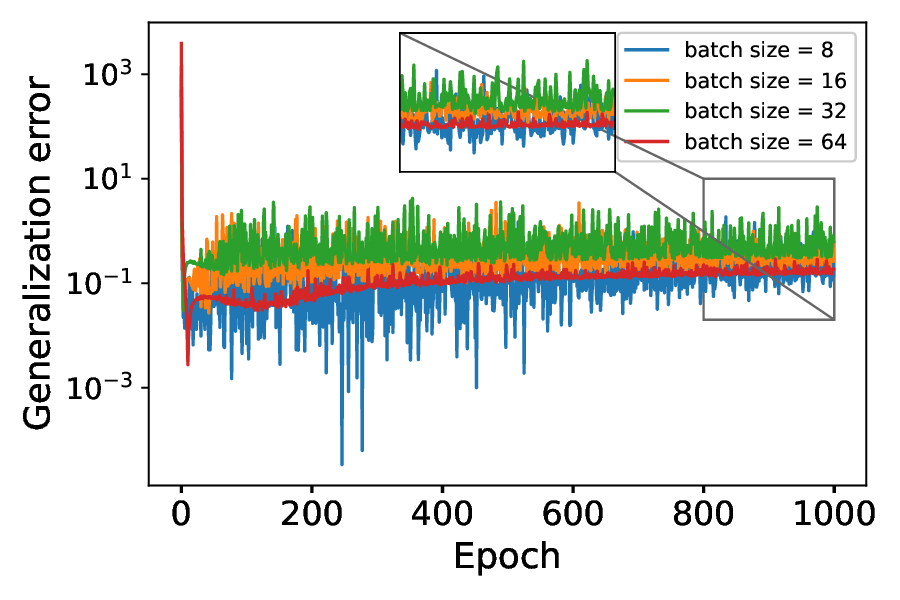}
    \end{subfigure}
    \caption{\textbf{MNIST.}\ \textbf{Left: }Stable rank $\srankV$ versus batch size.\   \textbf{Right: }Absolute value of generalization error. Here we fix $\mu_{V}=1$.}
    \label{fig:MNIST_stablerank_batch}
\end{figure}

\end{document}